\documentclass[nohyperref]{article}
\usepackage{arxiv}
\usepackage{natbib}
\usepackage{microtype}
\usepackage{graphicx}
\usepackage{subfigure}
\usepackage{booktabs} 

\usepackage{hyperref}
\usepackage{quiver}

\usepackage{amsmath}
\usepackage{amssymb}
\usepackage{mathtools}
\usepackage{amsthm}

\usepackage[capitalize,noabbrev]{cleveref}

\usepackage{amsmath,amsfonts,bm}

\def\eqref#1{equation~\ref{#1}}

\def\1{\bm{1}}

\def\vp{{\bm{p}}}
\def\vq{{\bm{q}}}

\def\vu{{\bm{u}}}
\def\vv{{\bm{v}}}
\def\vw{{\bm{w}}}

\DeclareMathAlphabet{\mathsfit}{\encodingdefault}{\sfdefault}{m}{sl}
\SetMathAlphabet{\mathsfit}{bold}{\encodingdefault}{\sfdefault}{bx}{n}

\def\gP{{\mathcal{P}}}

\def\gW{{\mathcal{W}}}
\def\gX{{\mathcal{X}}}

\DeclareMathOperator*{\argmax}{arg\,max}

\usepackage{hyperref}
\usepackage{url}
\usepackage[utf8]{inputenc} 
\usepackage[T1]{fontenc}    

\usepackage{bbold}
\usepackage{graphics}

\usepackage{pst-node}
\usepackage{tikz-cd}
\usepackage{enumitem}
\usepackage{wrapfig,color}

\usepackage{amsmath}
\usepackage{amssymb}
\usepackage{amsthm}
\usepackage{xspace}
\usepackage[normalem]{ulem}
\usepackage{graphicx}
\usepackage{epsfig}
\usepackage{ifpdf}
\usepackage{soul,hyperref}
\usepackage{latexsym}
\usepackage[mathscr]{euscript}
\usepackage{xspace}
\usepackage{color}
\usepackage{makeidx}
\usepackage{wrapfig,color}
\usepackage{stackrel}
\usepackage{algorithm, algorithmic}
 \usepackage[toc,page]{appendix}

\usepackage[all]{xy}

\usepackage{mathtools}

\usepackage{extarrows}

\usepackage{caption}

\theoremstyle{plain}
\newtheorem{theorem}{Theorem}[section]
\newtheorem{proposition}[theorem]{Proposition}

\newtheorem{corollary}[theorem]{Corollary}
\theoremstyle{definition}
\newtheorem{definition}[theorem]{Definition}

\theoremstyle{remark}
\newtheorem{remark}[theorem]{Remark}

\usepackage[textsize=tiny]{todonotes}
\newcommand{\propositionofref}{}
\newtheorem*{zpropositionof}{Proposition \propositionofref}
\newenvironment{propositionof}[1]
 {\renewcommand{\propositionofref}{#1}\zpropositionof}
 {\zpropositionof}

 \newcommand{\theoremofref}{}

\newcommand{\Int}{\mathbb{Z}}
\newcommand{\Real}{\mathbb{R}}
\newcommand{\RR}{\mathbb{R}}

\newcommand{\uu}{\vu}
\newcommand{\ww}{\vw}

\newcommand*{\rom}[1]{\expandafter\@slowromancap\romannumeral #1@}

\newcommand{\rank}{\mathrm{rank}}

\newcommand{\grid}{\text{\textsc{Grid}}}

\newcommand{\di}{\mathrm{d}_\mathrm{I}}

\newcommand{\I}{\mathcal{I}}

\newcommand{\gril}{\textsc{Gril}}

\newcommand{\cancel}[1]

\makeatletter
\newcommand{\colim@}[2]{%
  \vtop{\m@th\ialign{##\cr
    \hfil$#1\operator@font colim$\hfil\cr
    \noalign{\nointerlineskip\kern1.5\ex@}#2\cr
    \noalign{\nointerlineskip\kern-\ex@}\cr}}%
}
\newcommand{\colim}{%
  \mathop{\mathpalette\colim@{\rightarrowfill@\scriptscriptstyle}}\nmlimits@
}
\renewcommand{\varprojlim}{%
  \mathop{\mathpalette\varlim@{\leftarrowfill@\scriptscriptstyle}}\nmlimits@
}
\renewcommand{\varinjlim}{%
  \mathop{\mathpalette\varlim@{\rightarrowfill@\scriptscriptstyle}}\nmlimits@
}
\makeatother

\newcommand\norm[1]{\lVert#1\rVert}

\newcommand{\derosion}{\mathrm{d}_\mathcal{E}}
\newcommand{\dlanscape}{\mathrm{d}_\mathcal{L}}
\newcommand{\RKG}{\mathsf{rk}}
\newcommand{\Nat}{\mathbb{N}}

\definecolor{azure(colorwheel)}{rgb}{0.0, 0.5, 1.0}

\title{GRIL: A $2$-parameter Persistence Based Vectorization for Machine Learning}

\author{
Cheng Xin*\\
\texttt{{xinc}@purdue.edu} \\
\And
Soham Mukherjee*\\
\texttt{{mukher26}@purdue.edu} \\
\And
Shreyas N. Samaga\\
\texttt{{ssamaga}@purdue.edu} \\
\And
Tamal K. Dey \\
\texttt{tamaldey@purdue.edu} 
}

\begin{document}
\maketitle
\def\thefootnote{*}\footnotetext{Both are considered first authors 
}\def\thefootnote{\arabic{footnote}}
\begin{abstract}
$1$-parameter persistent homology, a cornerstone in Topological Data Analysis (TDA), studies the evolution of topological features such as connected components and cycles hidden in data.
 It has been applied to enhance the representation power of deep learning models, such as Graph Neural Networks (GNNs).
 To enrich the representations of topological features,  here we propose to study $2$-parameter persistence modules induced by bi-filtration functions. 
 In order to incorporate these representations into machine learning models, we introduce a novel vector representation called Generalized Rank Invariant Landscape {(\textsc{Gril})} for $2$-parameter persistence modules.  We show that this vector representation is $1$-Lipschitz stable and differentiable with respect to underlying filtration functions and can be easily integrated into machine learning models to augment encoding topological features. We present an algorithm to compute the vector representation  efficiently. We also test our methods on synthetic and benchmark graph datasets, and compare the results with previous vector representations of $1$-parameter and $2$-parameter persistence modules. Further, we augment GNNs with \textsc{Gril} features and observe an increase in performance indicating that \textsc{Gril} can capture additional features enriching GNNs. 
 We make the complete code for the proposed method available at \href{https://github.com/soham0209/mpml-graph}{https://github.com/soham0209/mpml-graph}.
\end{abstract}

\section{Introduction}


Machine learning models such as 
Graph Neural Networks (GNNs)~\citep{gnn,gnn2,kipf2017semi,gin_xu2018how} are well-known successful tools from the geometric deep learning community. 
Some recent research has indicated that the representation power of such models can be augmented by infusing topological information~\citep{Hofer2017DeepLW,GNN_graph_topology19,PersLay,togl}.
One way to do that is by applying \emph{persistent homology}, which is a powerful tool for characterizing the shape of data, rooted in the theory of algebraic topology. It has spawned the flourishing area of Topological Data Analysis.
The classical persistent homology, also known as, $1$-parameter \emph{persistence module}, has attracted plenty of attention from both theory and applications~\citep{edelsbrunner2010computational, TDA_quiver_rep_oudot,TDA_app_carlsson_2021,dey_wang_2022_book}.
In essence, a $1$-parameter persistence homology captures the evolution of some topological information within a topological space $\gX$ along an ascending filtration determined by a scalar function $\gX\to\RR$. It can be losslessly summarized by a complete discrete invariant such as a \emph{persistence diagram}, \emph{rank invariant} or \emph{barcode}.
In recent years, many works have successfully integrated persistence homology with machine learning models
\citep{Multipers_Kernel_Kerber,chen2019topological,PersLay, PLLay, gabrielsson2020topology, hofer2020graph, Zhao2020PersistenceEG, swenson2020persgnn, Carriere_Multipers_Images,Multipers_landscapes, bouritsas2020improving, togl, topologynet, todd, gefl, dowker}.

To further enhance the capacity of persistent homology, it is natural to consider a more general multivariate filtration function $\gX\to\RR^{d}$ for $d\geq 2$ in place of a real valued
function, and represent its topological information by multiparameter persistence modules. However, the structure of multiparameter persistence modules is much more complicated than $1$-parameter persistence modules.
In $1$-parameter case, the modules are completely characterized by what is called \emph{barcode} or \emph{persistence diagram}~\citep{Chazal:2009, InterleavingMulti_Lesnick2015}. Unfortunately,
there is no such discrete complete invariant which can summarize multiparameter persistence modules completely~\citep{Carlsson2009}. Given this limitation, building a useful vector representation from multiparameter persistence modules while capturing as much topological information as possible for machine learning models becomes an important but challenging problem.


To address this challenge, different kinds of vector representations have been proposed for $2$-parameter persistence modules~\citep{Multipers_Kernel_Kerber, Multipers_landscapes,Carriere_Multipers_Images}.
All these works are essentially based on the invariant called fibered (sliced) barcodes
~\citep{rivet}.
However, such representations capture as much topological information as determined by the well-known incomplete summary called rank invariant~\citep{Carlsson2009} which is equivalent to fibered barcodes.


In this paper, we propose a new vector representation to extend its expressive power in terms of capturing topological information from a $2$-parameter persistence
module:

\begin{itemize}
    \item We introduce \emph{Generalized Rank Invariant Landscape} (\textsc{Gril}), a new vector representation encoding richer information beyond fibered barcodes for $2$-parameter persistence modules, based on the idea of \emph{generalized rank invariant}~\citep{Kim2021GeneralPersisDiagramPosets} and its computation by zigzag persistence \cite{DKM22}. The construction of \textsc{Gril} can be viewed as a generalization of persistence landscape~\citep{bubenik2015statistical,Multipers_landscapes}, hence has more discriminating power.
    \item We show that this vector representation \textsc{Gril} is $1$-Lipschitz stable and differentiable with respect to the filtration function $f$, which allows one to build a topological representation as a machine learning model. 
    \item We propose an efficient algorithm to compute (\textsc{Gril}), demonstrate its use on synthetic and benchmark graph datasets, and compare the results with previous vector representations of $1$-parameter and $2$-parameter persistence modules. Specifically, we present results indicating that GNNs may improve when augmented with \gril{} features for graph classification task.
\end{itemize}

\section{Background}
In this section, we start with an overview of single and multiparameter persistence modules followed by formal definitions of basic concepts. Then we provide a high-level idea of how to construct our vector representation {\gril{}}.
For a more comprehensive introduction to persistence modules, we refer the
interested reader to~\citep{edelsbrunner2010computational, TDA_quiver_rep_oudot,TDA_app_carlsson_2021,dey_wang_2022_book}. 


The standard pipeline of $1$-parameter persistence module is as follows: Given a domain of interest $\gX$ (e.g. a topological space, point cloud data, a graph, or a simplicial complex) with a scalar function $f:\gX\to\RR$, one filters the domain $\gX$ by the sublevel sets $\gX_{\alpha}\triangleq\{x\in\gX\mid f(x)\leq \alpha\}$ along with a continuously increasing threshold $\alpha\in\RR$. The collection $\{\gX_{\alpha}\}$, which is called a \emph{filtration}, forms an increasing sequence of subspaces 
$\emptyset=\gX_{-\infty}\subseteq \gX_{\alpha_1}\subseteq\cdots\subseteq\gX_{+\infty}=\gX$. 
Along with the filtration, topological features appear, persist, and disappear over a collection of intervals. We consider $p$th homology groups $H_p(-)$ over a field, say $\Int_2$ ,
of the subspaces in this filtration, which results into a sequence of vector spaces.
These vector spaces are connected by inclusion-induced linear maps 
forming an algebraic structure
$0=H_p(\gX_{-\infty})\to H_p(\gX_{\alpha_1})\to\cdots\to H_p(\gX_{+\infty})$.
(see~\citep{AT_algtop}). 
This algebraic structure, known as $1$-parameter persistence module induced by $f$ and denoted as $M^f$, can be uniquely decomposed into a collection of atomic modules called interval modules, which completely characterizes the topological features in regard to the three behaviors--appearance, persistence, and disappearance of all $p$-dimensional cycles. This unique decomposition of a $1$-parameter persistence module is commonly summarized as a complete discrete invariant, \emph{persistence diagram}~\citep{edelsbrunner2000topological} or \emph{barcode}~\citep{Zomorodian2005}. 
Figure~\ref{fig:filter} (left) shows a filtration of a simplicial complex that induces
a $1$-parameter persistence module and its decomposition into bars.
\begin{figure}[htb]
    \centering
    \includegraphics[width=\columnwidth]{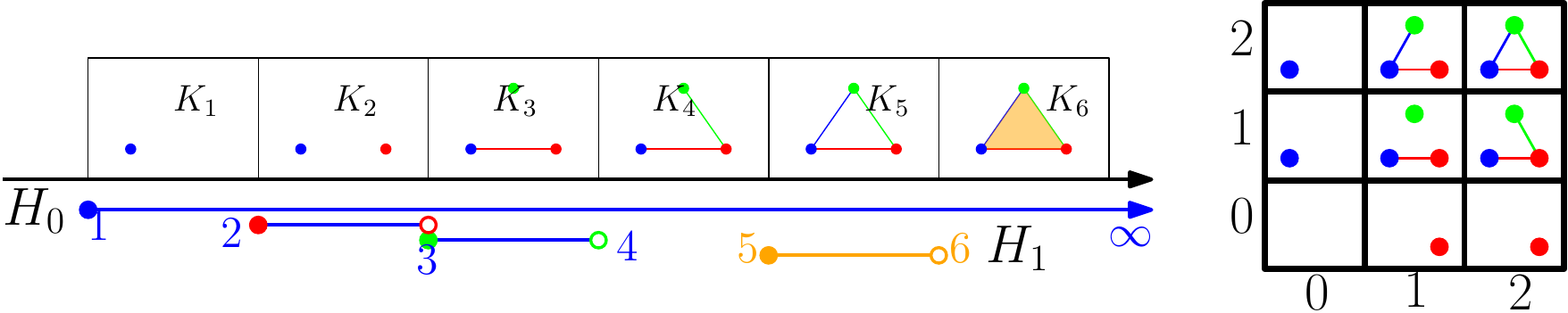}
    \caption{(left) $1$-parameter filtration and bars; (right) a $2$-parameter filtration inducing a $2$-parameter persistence module whose decomposition is not shown.}
    \label{fig:filter}
\end{figure}

Some problems in practice may demand tracking the topological information in a filtration that is not necessarily linear.
For example, in~\citep{Classification_Hepatic_Lesions}, $2$-parameter persistence modules are shown to be better for classifying hepatic lesions compared to $1$-parameter persistence modules.
In~\citep{keller_lesnick_willke_2018, todd}, a virtual screening system based on $2$-parameter persistence modules are shown to be effective for searching new candidate drugs.
In such applications, instead of studying a sequential filtration filtered by a scalar function, one may study a grid-filtration induced by a $\RR^2$-valued bi-filtration function $f:\gX\to \RR^2$ with $\RR^2$ equipped with partial order $\vu\leq\vv: u_1\leq v_1, u_2\leq v_2$; see Figure~\ref{fig:filter}(right) for an example of $2$-parameter filtration. Following a similar pipeline as the $1$-parameter persistence module, one will get a collection of vector spaces $\{M^f_\vu\}_{\vu\in\RR^2}$ indexed by vectors $\vu=(u_1, u_2)\in \RR^2$ and linear maps $\{M^f_{\vu\to\vv}: M^f_{\vu}\to M^f_\vv \mid \vu\leq\vv\in\RR^2\}$ for all comparable $\vu\leq\vv$. The entire structure $M^f$, in analogy to the $1$-parameter case, is called a $2$-parameter persistence module induced from $f$.
Unlike $1$-parameter case, there is no \emph{complete} discrete invariant like persistence diagrams or barcodes
that can losslessly summarize the whole structure of $2$-parameter persistence modules~\citep{Carlsson2009}. 
A good non-complete invariant for $2$-parameter persistence modules should characterize many non-isomorphic topological features, ideally as many as possible. At the same time, it should be stable with respect to small perturbations of filtration functions, which guarantees its important properties of continuity and differentiability for machine learning models.  Therefore, building a good summary in general for $2$-parameter persistence modules which is also applicable to machine learning models is an important and challenging problem.

We now formally define some of the concepts discussed above.

\begin{definition}[Simplicial Complex]\label{app_def:simplicial_comp}
An abstract simplicial complex is a pair $(V,\Sigma)$ where $V$ is a finite set and $\Sigma$ is a collection of non-empty subsets of $V$ such that if $\sigma \in \Sigma$ and if $\tau \subseteq \sigma$ then $\tau \in \Sigma$. 
A topological space $|(V,\Sigma)|$ can be associated with the simplicial complex which can be defined using a bijection $t\colon V \to \{1,2,\hdots,|V|\}$ as the subspace of $\RR^{|V|}$ formed by the union $\bigcup \limits_{\sigma \in \Sigma}h(\sigma)$, where $h(\sigma)$ denotes the convex hull of the set $\{e_{t(s)}\}_{s \in \sigma}$, where $e_i$ denotes the standard basis vector in $\RR^{|V|}$.
\end{definition}

\begin{definition}[Simplicial Filtration]\label{2_sim_filt}
A $d$-parameter \emph{simplicial filtration} over $\RR^d$ for some $d\in \Int_+$ is a collection of simplicial complexes $\{X_{\mathbf{u}}\}_{\mathbf{u} \in \RR^d}$ with inclusion maps $X_{\mathbf{u}} \xhookrightarrow{} X_{\mathbf{v}}$ for $\mathbf{u}\leq \mathbf{v}$, that is, $u_1\leq u_2$ and $v_1 \leq v_2$ where $\mathbf{u}=(u_1,u_2)$ and $\mathbf{v}=(v_1,v_2)$. When $d=2$, it is also called bi-filtration,
\end{definition}

\begin{definition} [Persistence Module]
    A $d$-parameter Persistence Module is a collection of vector spaces $\{X_{\mathbf{u}}\}_{\mathbf{u}\in \RR^d}$ indexed by $\RR^d$,
     together with a collection of linear maps
    $\{M_{\uu\to\vv}:M_{\mathbf{u}} \to M_{\mathbf{v}}\mid \uu\leq \vv \in \RR^d \}$
    such that $M_{\vv\to\ww}=M_{\vv\to\ww } \circ M_{\uu\to\vv}, \forall \uu\leq\vv\leq \ww$. 
\end{definition}
\begin{remark}
    In this paper, we study $1$ and $2$-parameter persistence modules for $d=1,2$. Each $M_\uu$ is the homology vector space of $X_\uu$ in a simplicial (bi-)filtration. And each $M_{\vv\to\ww}$ is the induced linear map from the inclusion $X_\uu\hookrightarrow  X_\vv$. 
\end{remark}

We now define the notion of an interval in $\RR^2$. In the definition, we shall make use of the standard partial order on $\RR^2$, i.e., $\vu \leq \vv$ if $u_1 \leq v_1$ and $u_2 \leq v_2$ for $\vu=(u_1, u_2)$ and $\vv=(v_1, v_2)$.  
\begin{definition}[Interval] A connected subset $\emptyset \neq I \subseteq {\RR^2}$ is an \emph{interval} if $\forall \vu\leq\vv\leq\vw, [\vu\in I, \vw\in I]\implies [\vv\in I]$.

\label{def:interval}
\end{definition}

    

We also give the definition of a zigzag filtration and the zigzag persistence module
induced by it as follows:
\begin{definition}[Zigzag filtration]
A \emph{zigzag filtration} is a sequence of simplicial
complexes 
where both insertions and deletions of simplices are allowed, the possibility of which we indicate with double arrows:
$$X_0\leftrightarrow X_1\leftrightarrow \cdots \leftrightarrow X_n={\mathcal X}.$$
Applying homology functor on such a filtration we get a zigzag persistence module
that is a sequence of vector spaces connected either by forward or backward
linear maps:
$$
H_*(X_0)\leftrightarrow H_*(X_1)\leftrightarrow \cdots\leftrightarrow H_*(X_n).
$$
\label{def:zigzag}
\end{definition}

We end this section by providing an overview of some high-level ideas for constructing our vector representation {\gril{}}.
\paragraph{Overview:}
Our approach computes a \emph{landscape function} over the $2$-parameter domain and then
vectorizes it. At this high level, this is similar to the approach in~\citep{Multipers_landscapes}. However,
the landscape function we construct is much more general and thus potentially has the power of
capturing more topological information. In particular, we use the concept of \emph{generalized
rank invariant} introduced in~\citep{Kim2021GeneralPersisDiagramPosets}, which indeed
generalizes the traditional rank invariant used in~\citep{Multipers_landscapes}. As opposed to simple
rank invariant which is defined over rectangles, generalized ranks are defined over their
generalizations called \emph{intervals}.
We define it more formally in section~\ref{sec:GRIL} below. 

One difficulty facing the use of the generalized ranks in TDA was that its efficient computation was not known. Recently, in~\citep{DKM22}, the authors showed that 
generalized ranks
for intervals in $2$-parameter persistence modules can be obtained by considering 
a persistence module supported on a linear poset induced by the boundary of the interval in question. 
However, this linear poset is not totally ordered as in $1$-parameter persistence, and
thus gives rise to what is called \emph{zigzag persistence}~\citep{zigzag_carlsson_deSilva} where the inclusions
can both be in forward and backward directions unlike traditional $1$-parameter persistence
where they are only in forward directions;
With this result, computing generalized ranks efficiently boils down to computing zigzag persistence efficiently.
For this purpose, we use a recently discovered fast zigzag algorithm and its efficient implementation ~\citep{dey2022fast}\footnote{\url{https://github.com/taohou01/fzz}}.

Our method samples a subset of grid points from the $2$-parameter grid spanned by a given bi-filtration function, and
computes the landscape function values (Definition~\ref{Defn:GRIL}) at those points based on generalized ranks. For this,
the algorithm considers an expanding sequence of intervals which we call \emph{worms} centered at each point $\vp$ and computes generalized rank over them
to determine the `width' of the maximal worm sustaining a chosen rank. This maximization
is achieved by a binary search over the sequence of worms centering $\vp$; section~\ref{sec:algs} describes this procedure. The widths, thus computed for each sample point, constitute the landscape function values which become the basis for our vector representation.

\section{Generalized Rank Invariant Landscape}
\label{sec:GRIL}

In this section, we introduce Generalized Rank Invariant Landscape, abbreviated as
 \gril{}, a stable and differentiable vector representation of $2$-parameter persistence modules. 

Let $M=M^f$ be a $2$-parameter persistence module induced by a filtration function $f$.
The restriction of $M$ to an interval $I$, denoted as $M|_I$, is the collection of vector spaces $\{M_\vu\mid \vu\in I\}$ along with linear maps $\{M_{\vu\to\vv}\mid \vu\leq\vv\in I)\}$. 
One can define the generalized rank of $M|_I$ as the rank of the canonical linear map from limit $\varprojlim M|_I$ to colimit $\varinjlim M|_I$ of $M|_I$:
\begin{equation*}
    \RKG^M(I) \triangleq \rank[\varprojlim M|_I \rightarrow     \varinjlim M|_I  ]
\end{equation*} 
A formal explanation of limit and colimit is beyond the scope
of this article;
we refer readers to~\citep{Saunders_Maclane_Cat_Theory} for their definitions and also the construction of the canonical limit-to-colimit map in category theory.
Intuitively, $\RKG^M(I)$ captures the number of \emph{independent}
topological features encoded in $M$ with the support over the entire interval $I$.
Specially, when $I=[\vu, \vv]\triangleq\{\vw\in\RR^2\mid \vu\leq\vw\leq\vv\}$ is a rectangle, $\varprojlim M|_I=M_\uu$ and $\varinjlim M|_I=M_\vv$. Then $\RKG^M(I)$ equals the traditional rank of the linear map $M_{\vu\to\vv}$. 
~\\
\begin{remark}
An interesting property of the generalized rank invariant is that 
its value over a larger interval is less than or equal to its value over any interval contained inside the larger interval. Formally, $I\subseteq J \implies  \RKG^M(I)\geq \RKG^M(J)$.
We implicitly use this \emph{monotone} property in the definition of \gril{}.
\label{rem:monotone}
\end{remark}

The basic idea of {\sc Gril} is to consider a collection of generalized ranks $\{\RKG^M(I)\}_{I\in\gW}$ over some covering set $\gW$ on $\Real^2$,
which is called a \emph{generalized rank invariant}
of $M$ over $\gW$. 


Let $\boxed{\vp}_\delta\triangleq\{\vw: \|\vp-\vw\|_\infty\leq \delta\}$ be 
the $\delta$-square centered at $\vp$ with side $2\delta$.
For given $\vp\in\RR^2, \ell\geq 1, \delta > 0$, we define an \emph{$\ell$-worm} $\boxed{\vp}_\delta^\ell$ to be the union over all $\delta$-squares $\boxed{\vq}_\delta$ centered at some point $\vq$ on the off-diagonal line segment $\vp+\alpha\cdot(1, -1)$ with $|\alpha|\leq (\ell-1)\delta$.   
See Figure~\ref{fig:lworm_examples} for an illustration.
\begin{figure}[!ht]
    \centering
    \includegraphics[width=0.8\columnwidth]{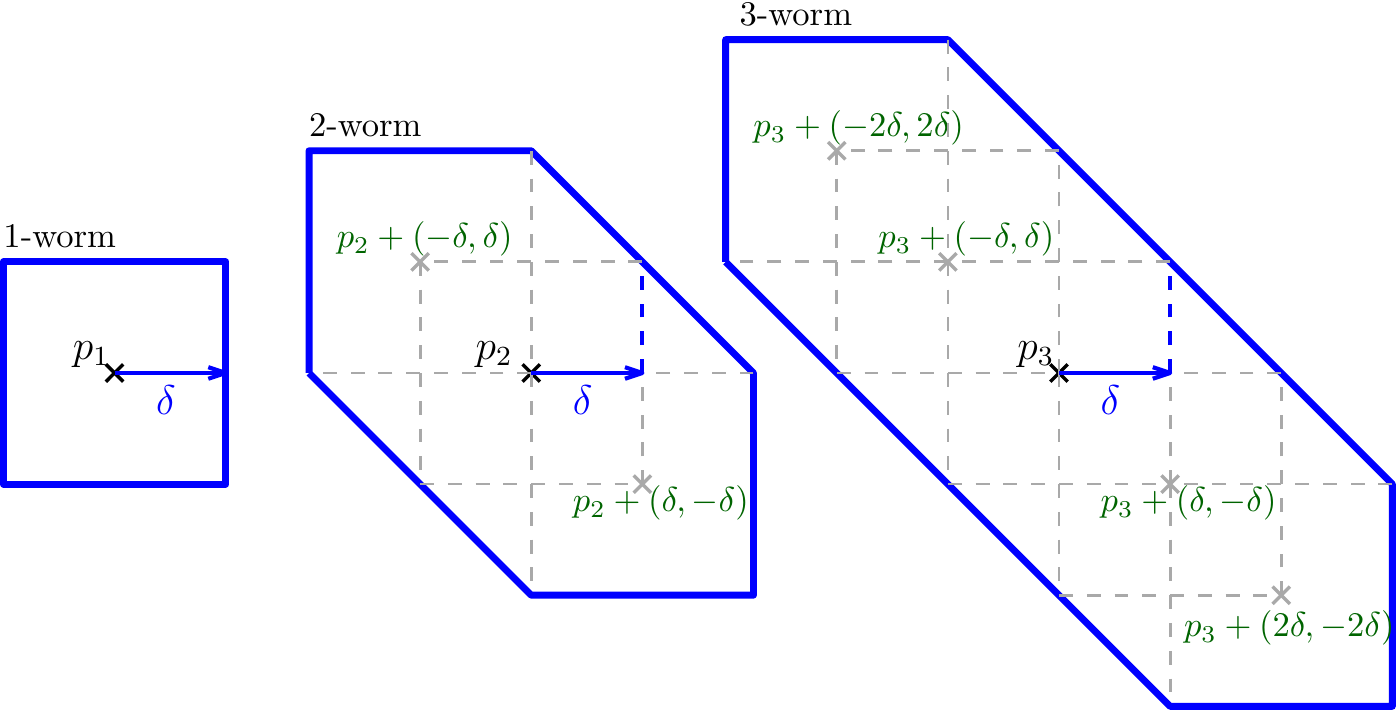}
    \caption{Examples of three $\ell$-worms with $\ell=1,2,3$.}
    \label{fig:lworm_examples}
\end{figure}

\begin{figure*}[!hbt]
    \centering
    \includegraphics[width=\textwidth, page=3]{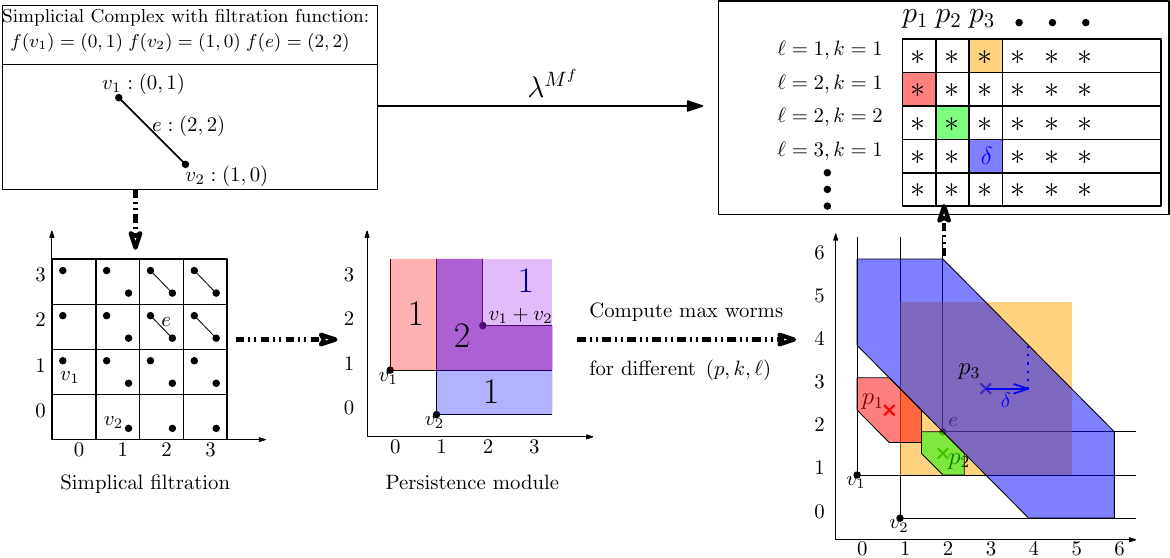}
    \caption{The construction starts from a simplicial complex with a bi-filtration function as shown on the top left. The simplicial complex consists of two vertices connected by one edge. Based on the bi-filtration, a simplicial bi-filtration can be defined as shown on the bottom left. On the mid bottom, a $2$-parameter persistence module is induced from the above simplicial filtration. If we check the dimensions of the vector spaces on all points of the plane, there are  $1$-dimensional vector spaces on red, blue and light purple regions. On the $L$-shaped dark purple region, the vector spaces have dimension $2$. For this $2$-parameter persistence module, we calculate $\lambda^{M^f}(\vp,k,\ell)$ for all tuples $(\vp,k,\ell)\in \gP\times K\times L$ to get our \gril{} vector representation. By Definition~\ref{Defn:GRIL} the value $\lambda^{M^f}(\vp,k,\ell)$ corresponds to the width of the maximal $\ell$-worm on which the generalized rank is at least $k$. On the bottom right, the interval in red is the maximal $2$-worm for $\lambda^{M^f}(\vp_1,k=1,\ell=2)$.
    The green interval is the maximal $2$-worm for $\lambda^{M^f}(\vp_2,k=2,\ell=2)$.
    The yellow square is the maximal $1$-worm for  $\lambda^{M^f}(\vp_3, k=1,\ell=1)$, and the blue interval is the maximal $3$-worm for $\lambda^{M^f}(\vp_3, k=1,\ell=3)$.
    Finally, on the top right, we have our {\gril{}} vector representation $\lambda^{M^f}$ which is a collection of vectors. Each vector corresponding to a different $\ell$ and $k$ consists of values as the width of maximal worms at each center point $\vp$. As an example, the blue one on the last vector at position $p_3$ has value $\delta$ which is the width of the blue worm.
    }
    \label{fig:working_exp}
\end{figure*}

Formally, 
\begin{equation*}
\boxed{\vp}^\ell_\delta\triangleq\bigcup_{\substack{\vq=\vp+(\alpha, -\alpha)\\ |\alpha|\leq (l-1)\delta}}{\boxed{\vq}_\delta}
\end{equation*}
We call $\vp$ the \emph{center point} and $\delta$ the \emph{width} of the $\ell$-worm $\boxed{\vp}^\ell_\delta$.
As a special case,
when $\ell=1$, $\boxed{\vp}^1_\delta=\boxed{\vp}_\delta$ is just the
$\delta$-square with side $2\delta$. 

We choose $\gW$ to be a set of \emph{Worms} defined as follows:
\begin{equation*}
    \gW\triangleq\left \{W=\boxed{\vp}_\delta^\ell \mid \delta>0, \ell\geq 1, \vp\in\RR^2 \right \} 
\end{equation*}

Now we are ready to define the main construct in this
paper which uses the monotone property of generalized rank
mentioned in Remark~\ref{rem:monotone}.
\begin{definition}[Generalized Rank Invariant Landscape (\textsc{Gril})]\label{Defn:GRIL}
    

For a persistence module $M$, the \emph{Generalized Rank Invariant Landscape (\textsc{Gril})} of $M$ is a function  $\lambda^M:\RR^2\times\Nat_+\times\Nat_+\rightarrow \RR$ defined as
\begin{equation}
\lambda^M(\vp, k, \ell)\triangleq \sup_{\delta\geq 0} \{\RKG^M({\boxed{\vp}^\ell_\delta)}\geq k\}. 
\label{eq:gril}
\end{equation}
\end{definition}
We can see from the definition that given a persistence module $M$, a point $\vp$, a rank $k$ and $\ell$, the value of \gril{} $(\lambda^M(\vp, k, \ell))$ is, in essence, the width $\delta$ of the "maximal" $\ell$-worm $W=\boxed{\vp}_\delta^{\ell}$ centered at $\vp$ such that the value of the generalized rank over $W$ is greater than or equal to $k$. See Figure \ref{fig:working_exp} bottom right for some examples of maximal worms. 

It turns out that, \gril{} as an invariant is equivalent to the generalized rank invariant over $\gW$.

\begin{proposition}\label{prop:equivalance}
\gril{} is equivalent to the generalized rank invariant over $\mathcal{W}$. Here the equivalence means bijective reconstruction from each other.
\end{proposition}
\begin{proof}\label{proof:equivalance}
    Constructing \gril{} from generalized rank invariant over $\gW$ is immediate from the definition of \gril{}.
    
    On the other direction, for any $\vp, \delta, \ell$, the generalized rank $\RKG^M(\boxed{\vp}_\delta^\ell)$ can be reconstructed by \gril{} as follows: 
    \begin{equation}
    \RKG^M(\boxed{\vp}_\delta^\ell)=\argmax_{k} \{\lambda(\vp, k, \ell)\geq \delta \}
    \end{equation}
    It is not hard to check that, this construction, combined with the construction of persistence landscape, gives a bijective mapping between generalized rank invariants over $\gW$ and \gril{}s. 
\end{proof}
    


See Figure~\ref{fig:working_exp} for an illustration of the overall pipeline of our construction of $\lambda^M$ starting from a filtration function on a simplicial complex.
 Figure~\ref{fig:MP_Signatures} shows the discriminating power of \gril{} where we see that \gril{} can differentiate between shapes that are topologically non-equivalent.

\begin{figure*}[ht]
    \centering
    \includegraphics[width=\textwidth]{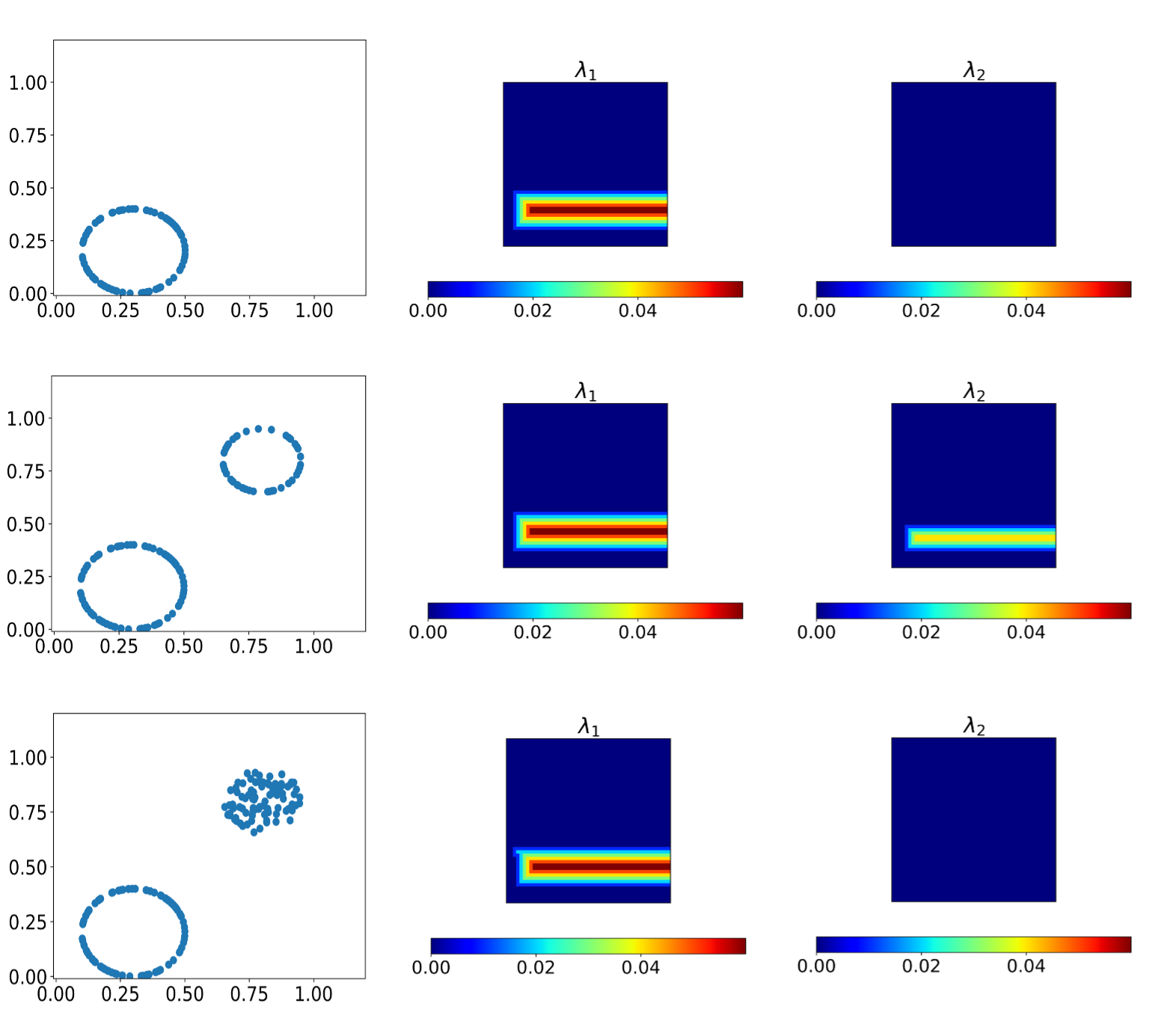}
    \caption{\gril{} as a topological discriminator: each row shows a point cloud $P$, \gril{} value heatmap for ranks $k=1$ and $k=2$ in homology of degree $1$ named as $\lambda_1$ and $\lambda_2$ respectively. First Betti number ($\beta_1$) of a circle is $1$ which is reflected in $\lambda_1$ being non-zero. $\beta_1$ for two circles is $2$ which is reflected in both $\lambda_1$ and $\lambda_2$ being non-zero. Similarly, $\beta_1$ of a circle and disk together is $1$ which is reflected in $\lambda_1$ being non-zero but $\lambda_2$ being zero for this point cloud.}
    \label{fig:MP_Signatures}
\end{figure*}

\subsection{Stability of GRIL}

\label{sec:stability}

An important property of \textsc{Gril} is its \emph{stability property}
which makes it immune to small perturbations of the input bi-filtration while still retaining the ability to characterize topologies. 
We will show \gril{} is stable with respect to input filtrations. 
We begin by defining two (pseudo)metrics,  \emph{interleaving distance} and \emph{erosion distance}, which we will use to characterize the stability property.
\begin{definition}[Morphism]\label{def:morphism}
Given two persistence modules $M$ and $N$, a \emph{morphism} from $M$ to $N$, $f:M\to N$, is a collection of linear maps $\{f_\uu: M_\uu\to N_\uu\}_{\uu\in \Real^2}$ such that $f_\uu\circ N_{\uu\to\vv}=M_{\uu\to\vv}\circ f_\vv, \forall \uu\leq \vv$. Essentially, the following diagram commutes.
\[\begin{tikzcd}
  {M_{\bm{u}}} && {M_{\bm{v}}} \\
  \\
  {N_{\bm{u}}} && {N_{\bm{v}}}
  \arrow["{M_{\bm{u} \rightarrow \bm{v}}}"{description}, from=1-1, to=1-3]
  \arrow["{N_{\bm{u} \rightarrow \bm{v}}}"{description}, from=3-1, to=3-3]
  \arrow["{f_{\bm{u}}}"{description}, from=1-1, to=3-1]
  \arrow["{f_{\bm{v}}}"{description}, from=1-3, to=3-3]
\end{tikzcd}\]
\end{definition}

\begin{definition}[Shift module]\label{def:shift_module}
Given a persistence module $M$ and $\varepsilon\in \Real$, we define the \emph{shift module} $M^{\leftarrow \varepsilon}$ through $M^{\leftarrow \varepsilon}_\uu \triangleq M_{\uu+\varepsilon}$ and $M^{\leftarrow \varepsilon}_{\uu\to\vv}\triangleq M_{\uu+\varepsilon \to \vv+\varepsilon}$. Here $\uu+\varepsilon=(\uu_1+\varepsilon, \uu_2+\varepsilon)$.
\end{definition}

The definition of interleaving distance is based on interleaving morphisms defined as follows:
\begin{definition}[$\varepsilon$-interleaving]\label{def:interleaving} 

For a pair of persistence modules $M, N$ and some $\varepsilon \geq 0$, an $\varepsilon$-\emph{interleaving} between $M$ and $N$ is a pair of morphisms $\phi:M\to N^{\leftarrow \varepsilon}$ and $\psi:N\to M^{\leftarrow \varepsilon}$ such that  $\forall \uu\in \mathbb{R}^2, M_{\uu \rightarrow \uu+2\varepsilon} =  \psi_{\uu+\varepsilon} \circ \phi_{\uu}$ and
$ N_{\uu \rightarrow \uu+2\varepsilon} =  \phi_{\uu+\varepsilon} \circ \psi_{\uu}$.
Essentially, the following diagram commutes.
If such interleaving exists, we say $M$ and $N$ are $\varepsilon$-interleaved.
\[\begin{tikzcd}
  \bullet & \bullet & {M_{\bm{u}}} && {M_{\bm{u}+\varepsilon}} && {M_{\bm{u}+2\varepsilon}} & \bullet & \bullet \\
  \\
  \bullet & \bullet & {N_{\bm{u}}} && {N_{\bm{u}+\varepsilon}} && {N_{\bm{u}+2\varepsilon}} & \bullet & \bullet
  \arrow[from=1-3, to=1-5]
  \arrow[from=1-5, to=1-7]
  \arrow[from=3-3, to=3-5]
  \arrow[from=3-5, to=3-7]
  \arrow[from=1-2, to=1-3]
  \arrow[from=1-7, to=1-8]
  \arrow[from=3-2, to=3-3]
  \arrow[from=1-1, to=1-2]
  \arrow[from=3-1, to=3-2]
  \arrow[from=3-7, to=3-8]
  \arrow[from=3-8, to=3-9]
  \arrow[from=1-8, to=1-9]
  \arrow["{\phi_{\bm{u}}}"{description, pos=0.7}, from=1-3, to=3-5]
  \arrow["{\psi_{\bm{u}+\varepsilon}}"{description, pos=0.7}, from=3-5, to=1-7]
  \arrow["{\phi_{\bm{u}+\varepsilon}}"{description, pos=0.7}, from=1-5, to=3-7]
  \arrow["{\psi_{\bm{u}}}"{description, pos=0.7}, from=3-3, to=1-5]
  \arrow["{M_{\bm{u} \rightarrow \bm{u}+2\varepsilon}}"{description}, curve={height=-18pt}, from=1-3, to=1-7]
  \arrow["{M_{\bm{u} \rightarrow \bm{u}+2\varepsilon}}"{description}, curve={height=18pt}, from=3-3, to=3-7]
\end{tikzcd}\]
\end{definition}

\begin{definition}[Interleaving Distance~\citep{Chazal:2009,InterleavingMulti_Lesnick2015}]\label{def:interleaving_distance}
For two persistence modules $M$ and $N$, \emph{interleaving distance} is defined as 
\[
\di(M, N)\triangleq \inf_{\varepsilon\geq 0}\{M \textrm{ and } N \textrm{ are } \varepsilon\textrm{-interleaved}\}.
\]
\end{definition}


We shall now look at a property of \textsc{Gril} that will help in proving the stability.

\begin{definition}
Given any interval $I$ and $\varepsilon\geq 0$, let $I^{+\varepsilon}$ be the \emph{$\varepsilon$-extension} of $I$ defined as:
\begin{equation}
    I^{+\varepsilon}\triangleq \bigcup_{\vp\in I} \boxed{\vp}_\varepsilon
\end{equation}
where $\boxed{\vp}_\varepsilon\triangleq\{\vq: ||\vp-\vq||_\infty \leq \varepsilon\}$ is the $\infty$-norm $\varepsilon$-neighbourhood of $x$.
\end{definition}
It follows from the definitions that:
\begin{proposition}
    $\left (\boxed{\vp}_\delta^\ell \right )^{+\varepsilon} \subseteq \boxed{\vp}_{\delta+\varepsilon}^\ell$.    
\end{proposition}

The erosion distance is a pesudometric on persistence modules defined by comparing modules' generalized rank invariants over all intervals in $\RR^2$ as follows: 
\begin{definition}[Erosion Distance~\citep{Patel_Gen_Pers_Diag,Kim2021GeneralPersisDiagramPosets}]
        Let $\mathbf{Int}(\RR^2)$ be the collection of all intervals in $\RR^2$. Let $M$ and $N$ be two persistence modules.
        The {\emph{erosion distance}} is defined as 
        \begin{equation*}
            \derosion(M, N)\triangleq \inf_{ \varepsilon\geq0}\{  \forall I \in\mathbf{Int}(\RR^2), \RKG^M(I)\geq \RKG^N(I^{+\varepsilon}) \text{ and } \RKG^N(I)\geq \RKG^M(I^{+\varepsilon}) \}.
        \end{equation*}
\end{definition}


    In order to better analyze the stability property of persistence landscape, we define a distance in a similar flavor as erosion distance 
    over our collection of worms $\gW$.
    \begin{definition}
    For $\mathcal{W}\triangleq\left \{\boxed{\vp}_\delta^\ell \mid \delta>0, l\in \Nat_+, \vp\in\RR^2 \right \}$, 
    define a distance $\derosion^\mathcal{W}$ as follows:
    \begin{flalign*}
         \derosion^\mathcal{W}(M, N)\triangleq     \inf_{\varepsilon\geq 0}  \{  \forall \boxed{\vp}_\delta^\ell\in\mathcal{W},  [ \RKG^M\left (\boxed{\vp}_\delta^\ell\right )\geq \RKG^N\left (\boxed{\vp}^\ell_{\varepsilon+\delta}\right ) \text{ and }  \RKG^N\left (\boxed{\vp}_\delta^l\right )\geq \RKG^M\left (\boxed{\vp}^\ell_{\varepsilon+\delta}\right )]  \}.
    \end{flalign*}
    \end{definition}
    
Now we are ready to state formally the stability property of {\gril}.
\begin{definition}
    For persistence module $M, N$ with {\gril{}s} $\lambda^M, \lambda^N$, define
    \begin{equation*}
        d_\mathcal{L}(M, N)\triangleq  ||\lambda^M-\lambda^N||_{\infty}.
    \end{equation*}
\end{definition}
    
\begin{proposition}
        $\dlanscape=\derosion^\mathcal{W}\leq \derosion$, where $\derosion$ is the erosion distance.    
\end{proposition}
    
    
    
    
    \begin{proof}
    $\derosion^\mathcal{W}\leq \derosion$ is obvious by definition. 
    
    To show $\dlanscape\leq \derosion^\mathcal{W}$. 
    Given two persistence modules $M, N$, assume  $\derosion^\I (M, N)=\varepsilon$. For fixed $\vp, k, \ell$, let $\lambda^M(\vp, k,\ell)=\delta_1$ and $\lambda^N(\vp, k, \ell)=\delta_2$. Without loss of generality, assume $\delta_2\geq \delta_1$. We want to show that $\delta_2 - \delta_1 \leq \varepsilon$. 
    By the construction of $\derosion^\mathcal{W}$, we know that for any $\alpha>0$, $k > 
    \RKG^N(\boxed{\vp}^\ell_{\delta_1+\alpha}(x))\geq \RKG^M(\boxed{\vp}^\ell_{\delta_1+\varepsilon+\alpha}(x))$. One can get $\delta_1+\varepsilon+\alpha > \delta_2\implies \varepsilon+\alpha > \delta_2-\delta_1$. By taking $\alpha\rightarrow 0$, we have $\delta_2-\delta_1\leq \varepsilon$. 
    

    To show $\derosion^\mathcal{W} \leq \dlanscape$. Let $\dlanscape(M, N)=\delta$. For any $I=\boxed{\vp}_\varepsilon^\ell\in\I$, we want to show that $\RKG^M(\boxed{\vp}_\varepsilon^\ell)\geq \RKG^N(\boxed{\vp}_{\varepsilon+\delta}^\ell)$ and $\RKG^N(\boxed{\vp}_\varepsilon^\ell)\geq \RKG^M(\boxed{\vp}_{\varepsilon+\delta}^\ell)$. We prove the first inequality. The second one can be proved in a similar way. Let $k=\RKG^N(\boxed{\vp}_{\varepsilon+\delta}^\ell)$, then $\lambda^N(\vp, k, \ell)\geq \varepsilon+\delta$. By the assumption $\dlanscape(M, N)=\delta$, we know that  $\lambda^N(\vp, k, \ell)\geq \varepsilon$, which implies $\RKG^M(\boxed{\vp}_\varepsilon^\ell)\geq k=\RKG^N(\boxed{\vp}_{\varepsilon+\delta}^\ell)$.
    \end{proof}

\begin{proposition}[Stability]\label{prop:stability}
    Given two filtration functions $f, f':\gX\to\Real^2$, 
    \[
        \left \|\lambda^{M^f}-\lambda^{M^{f'}}\right \|_\infty \leq \left \| f-f'\right \|_\infty
    \]
    %
\end{proposition}

\begin{proof}
    Let $M^f$ and $M^{f'}$ be the persistence modules derived by $f$ and $f'$ respectively. Then, we have the following chain of inequalities:
    \begin{align*}
        \norm{\lambda^{M^f}-\lambda^{M^{f'}}}_{\infty} = & \dlanscape(M^f, M^{f'}) \\
        \leq & \derosion(M^f, M^{f'}) \\
        \leq & \di(M^f, M^{f'}) \\
        \leq & \norm{f-f'}_{\infty}
    \end{align*}
    where $\di(M^f, M^{f'})$ is the interleaving distance.
    The second to last inequality is shown in \citep{Kim2021GeneralPersisDiagramPosets}.
    \end{proof}


\subsection{Differentiability of \gril{}}
When $\gX$ is a finite space (e.g. finite simplicial complex (see Definition \ref{app_def:simplicial_comp})) with $|\gX|=n$ then, any $f \colon \gX \to \RR^2$ can be represented as a vector in $\RR^{2n}$. From Proposition~\ref{prop:stability}, we can get Lipschitz continuity property for {\gril{}}.

\begin{proposition}[Lipschitz continuous]
\label{prop:gril_Lipschitz}
For a finite space $\gX$ with $|\gX|=n$ and fixed $k,\ell,\vp$, the function $\Lambda^{k,\ell}_{\vp}:\RR^{2n}\to \RR$ given by $\Lambda^{k,\ell}_{\vp}(f)=\lambda^{M^f}(k,\ell,\vp)$ is Lipschitz continuous.
\end{proposition}
\begin{proof}
     Given filtration functions $f, f'$ and their corresponding vector representations $v_f, v_{f'} \in \RR^{2n}$, one can check that $\norm{f-f'}_\infty \leq 2\norm{v_f - v_{f'}}_\infty \leq 2\norm{v_f - v_{f'}}_2$. Combining this with Proposition~\ref{prop:stability}, 
     we get that $\Lambda^{k,\ell}_{\vp}$ is Lipschitz continuous with respect to the underlying filtration functions. 
\end{proof}

By Rademacher's theorem~\citep{Evans_Measure_Theory}, we have $\Lambda^{k,\ell}_{\vp}$, as a Lipschitz continuous function, being differentiable almost everywhere.

\begin{corollary}\label{cor:gril_differentiable}
   $\Lambda^{k,\ell}_{\vp}$ is differentiable almost everywhere.
\end{corollary}

The differentiability of $\Lambda_\vp^{k,\ell}$
in Corollary~\ref{cor:gril_differentiable} stands for the existence of directional derivatives. 
But the existence of a "steepest" direction 
might not be unique. 
Here we propose a method to compute one specific steepest direction based on Proposition~\ref{thm:derivative} (see the proof and an experiment testing our method on a synthetic dataset in Appendix~\ref{asec:diff}). 



\begin{proposition}\label{thm:derivative}
    Consider the space of all filtration functions $\{f:\gX\to \RR^2\}$ on a finite space $\gX$ with $|\gX|=n$, which is equivalent to $\RR^{2n}$. 
    For fixed $k, \ell, \vp$, 
   there exists a measure-zero subset $Z\subseteq \RR^{2n}$ such that for any $f\in \RR^{2n}\setminus Z$ satisfying the following generic condition: $\forall x\neq y\in \gX, f(x)_1\neq f(y)_1, f(x)_2\neq f(y)_2$, 
    there exists an assignment $s:\gX\to \{\pm 1,0,\pm\ell\}^2$
    such that 
    \begin{align*}
        \nabla_s \Lambda_\vp^{k,\ell}(f)\triangleq & \lim_{\alpha\to 0}\frac{\Lambda_\vp^{k,\ell}(f+\alpha s)-\Lambda_\vp^{k,\ell}(f)}{\alpha\|s\|_{\infty}}\\
        = & \max_{g\in\gX}\nabla_g\Lambda_\vp^{k,\ell}(f).
    \end{align*}
\end{proposition}

\section{Algorithm}
\label{sec:algs}
We present our algorithm to compute \gril{} in this section.

In practice, we choose center points $\vp$ from some finite subset $\gP\subset\RR^2$, e.g. a finite uniform grid in $\RR^2$, and consider $k\leq K, \ell\leq L$ for some fixed $K, L\in \Nat_+$.  
Then, \gril{} $\{\lambda^M( \vp, k,\ell)\}$
can be viewed as a vector of dimension $|\gP|\times K\times L$.

The high-level idea of the algorithm is as follows: Given a bi-filtration function $f:\gX\to\RR^2$, 
for each triple $(\vp, k, \ell)\in \gP\times  K \times L$, we compute 
$\lambda^{M^f}(\vp, k, \ell)=\sup_{\delta\geq 0} \{\RKG^{M^f}({\boxed{\vp}^\ell_\delta)}\geq k\}$.
In essence, we need to compute the maximum width over worms on which 
the generalized rank is at least $k$. In order to find the value of this width, we use binary search. We compute generalized rank $\RKG^{M^f}\left({\boxed{\vp}^\ell_\delta}\right)$ by applying the algorithm proposed in \citep{DKM22}, which uses zigzag persistence on a boundary path. This zigzag persistence is computed efficiently by a recent algorithm proposed in~\citep{dey2022fast}. We denote the sub-routine to compute generalized rank over a worm by \textsc{ComputeRank} in algorithm \ref{alg:CompGRIL} mentioned below. 
{\sc{ComputeRank}}($f, I$) takes as input a bi-filtration function $f$ and an interval $I$,  and outputs generalized rank over that interval. 
In order to use the algorithm proposed in \citep{DKM22}, the worms need to have their boundaries aligned with a grid structure defined on the range of $f$. Thus, we normalize $f$ to be in the range $[0,1]\times[0,1]$, define a grid structure on $[0,1]\times[0,1]$ and discretize the worms. 
Let $\grid = \{ \left (\frac{m}{M}, \frac{n}{M} \right ) \mid m, n \in \{0,1,\hdots, M \}\}$ for some $M\in \Int_+$. We denote the grid resolution as $\rho\triangleq 1/M$. We 
take the set of center points $\mathcal{P}\subseteq \grid$ as a uniform subgrid of $\grid$.
We consider the discrete worms for $\vp\in\gP, \delta=d\cdot\rho, d\in\Int_{\geq 0}$ as follows:
\begin{equation}\label{eq:discrete_worm}
    \widehat{\boxed{\vp}}^\ell_\delta\triangleq \bigcup_{\substack{\vq=\vp+(\alpha, -\alpha)\\  |\alpha|\leq (l-1)\delta \\\vq\in\grid{}}}{\boxed{\vq}_\delta}.
\end{equation}
Essentially, a discrete $\ell$-worm $\widehat{\boxed{\vp}}^\ell_\delta$ centered at $\vp$ with width $\delta$ is the union of $2\ell-1$ squares with width $\delta$ centered at $\vp \pm (c\delta, -c\delta)$ for $c \in \{0,1,\ldots, \ell-1\}$ along with the intermediate staircases between two consecutive squares of step-size equal to \emph{grid resolution} ($\rho$). Figure \ref{fig:discre_l_worm} (middle) shows the discretization of a $2$-worm. This construction is sensitive to the grid resolution. 

Now all such discrete worms $\widehat{\boxed{\vp}}$ 
are intervals whose boundaries are aligned with the $\grid$.
We apply the procedure {\sc{ComputeRank}}$(f, I)$ to compute $\RKG^{M^f}(I)$ for $I={\widehat{\boxed{\vp}}^\ell_\delta}$. Denote
\begin{equation}\label{eq:approximate_lambda}
    \hat{\lambda}^{M^f}(\vp, k, \ell)=\sup_{\delta\geq 0} \{\RKG^{M^f}({\widehat{\boxed{\vp}}^\ell_\delta)}\geq k\}.
\end{equation}
\begin{remark}\label{rmk:discrete_approximate}
    One can observe that 
    \[
        \lambda^{M^f}(\vp, k, \ell)\leq \hat{\lambda}^{M^f}(\vp, k, \ell)\leq \lambda^{M^f}(\vp, k, \ell)+ \rho
    \]
    Therefore, we compute $\hat{\lambda}$ as an approximation of $\lambda$ in practice.
\end{remark}

The pseudo-code is given in Algorithm~\ref{alg:CompGRIL}. The algorithm is described in detail in Appendix~\ref{sec:algorithm}.
\begin{figure}[!htb]
    \centering
    \includegraphics[width=0.7\textwidth]{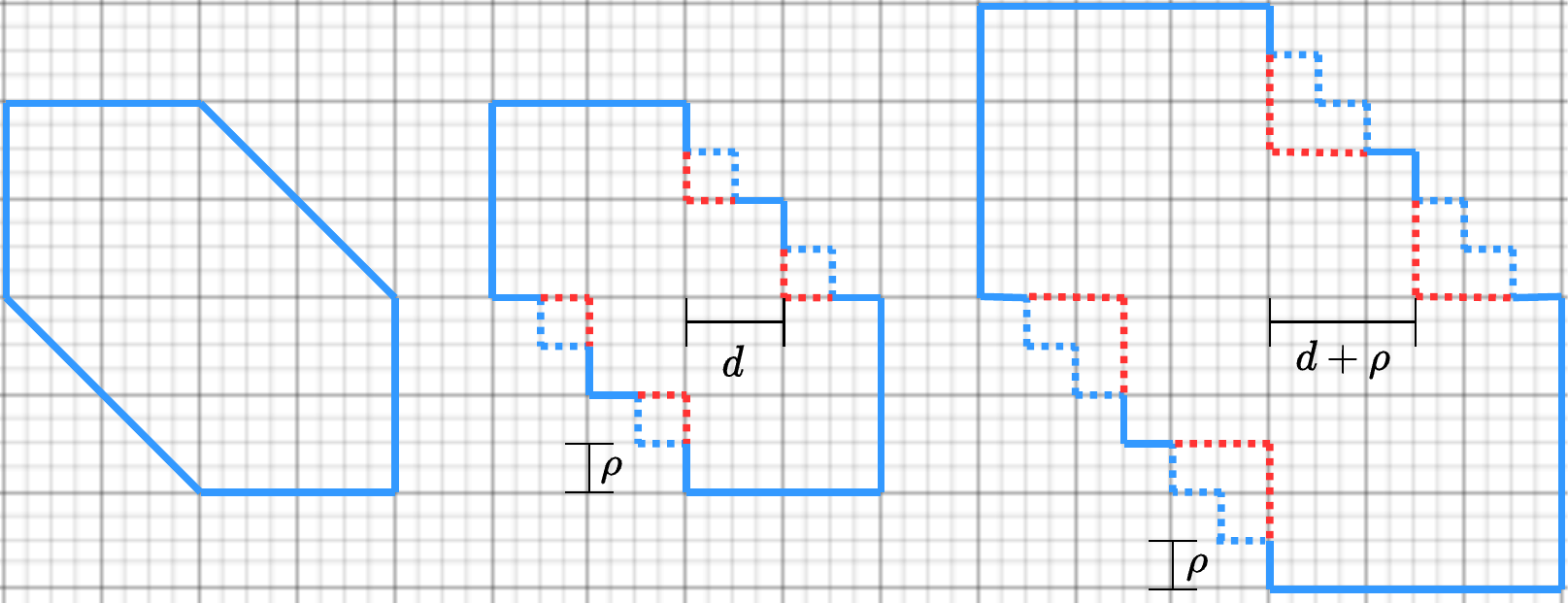}
    \caption{A $2$-worm, discretized $2$-worm and expanded discretized $2$-worm; $\rho$ denotes grid resolution. 
    The blue dotted lines show the intermediate staircase with step-size $\rho$. The red dotted lines form parts of the squares with size $d$ which are replaced by the blue dotted lines in the worm. 
    The last figure shows the expanded $2$-worm with red and blue dotted lines. The expanded $2$-worm has width $d + \rho$ which is the one step expansion of the worm with width $d$.}
    \label{fig:discre_l_worm}
\end{figure}

\begin{algorithm}[h!]
\caption{\textsc{Compute\gril}}\label{alg:CompGRIL}
\begin{algorithmic}
   \STATE {\bfseries Input:} $f: \text{Bi-filtration function}, \ell \geq 0, k\geq 1, \vp \in \mathcal{P}\subseteq \grid$, $\rho$: grid resolution
   \STATE {\bfseries Output:} $\hat{\lambda}(\vp, k, l)$:  \gril{} value at a point $\vp$ for  fixed $k$ and $\ell$
   \STATE {\bfseries Initialize:} $d_{min} \gets \rho, d_{max} \gets 1$, $\lambda \gets 0$ 
   \WHILE{$d_{min} \leq d_{max}$}
   \STATE $d \gets  (d_{min} + d_{max}) / 2 $.
   \STATE $I \gets \widehat{\boxed{\vp}}^\ell_d$ 
   \STATE $r \gets$ \textsc{ComputeRank($f, I$)}
   
   \IF {$r \geq  k$}
        \STATE $\lambda \gets d$
        \STATE $d_{min} \gets d+\rho$
    \ELSE{}
        \STATE $d_{max} \gets d-\rho$ 
    \ENDIF
    \ENDWHILE{}
    
\RETURN {$\lambda$}

\end{algorithmic}
\end{algorithm}   

\textbf{Time complexity.} 
Assuming a grid with $t$ nodes and a bi-filtration of
a complex with $n$ simplices on it, one can observe that each probe in the binary search
takes $O(n^\omega)$ time 
where $\omega<2.37286$ is the matrix multiplication exponent~\citep{matrix_exponent}. This is because each probe generates a zigzag filtration of length $O(n)$ with $O(n)$ simplices. Therefore, the binary search takes $O(n^\omega\log t)$ time giving a total time complexity of $O(tn^{\omega}\log t))$ that accounts for $O(t)$ worms.

\textbf{Speeding up the implementation.} 
In implementation, we use some observations that help run \textsc{ComputeGril} more efficiently in practice. When computing \gril{} for $k=1,2,\hdots , n$, we use the monotone property described in Remark \ref{rem:monotone} to reduce the scope of the binary search for successive values of $k$. For example, the value of \gril{} for $k$ is always greater than or equal to the value of \gril{} for $k+1$. Thus, we can reduce the scope of the binary search while computing for $k+1$ by setting the maximum in the binary search to be the value of \gril{} at $k$. Further, we store the values of rank for a given width $d$ while computing the value of \gril{} for a $k$. This information can be reused in later computations. For example, we store the values of generalized ranks of worms for different values of $d$ at a center point $\vp$ during the binary search for, say $k=k_0$. We use this information for successive binary searches for all $k > k_0$ and save on the zigzag persistence computation for those values of $d$. While computing zigzag persistent, along with the barcode for $0$th homology group, the barcode for $1$st homology group is also computed. We store this information and reuse it while computing \gril{} values for $1$st homology group. These observations reduce the total number of zigzag persistence computations to a significant extent resulting in reducing the total computational time.

\section{Experiments}\label{sec:experiment}
Our method \gril{} exploits generalized rank invariant whereas existing methods exploit rank invariant which is equivalent to fibered barcode. Although both invariants are known to be incomplete for multiparameter persistence as any other discrete invariant, the generalized rank invariant is more informative in theory. 
Our experiments support this theoretical hypothesis in practice to some extent as we obtain better accuracy for all cases in Table~\ref{tb:toy_dataset} and  $13$ out of $20$ cases in Table~\ref{tab:graph_acc} in comparison to existing methods applying some form of fibered barcodes.
We perform experiments
on synthetic datasets as well as graph benchmark datasets. On these datasets, we define a bi-filtration and compute \gril{} values $\lambda(\vp, k, \ell)$ for
$\ell=2$ and for
each $k \in \{1,2,\hdots,5\}$ where $\vp$ is chosen over a uniform subgrid. Some datasets require a finer resolution for capturing meaningful information while for others, finer resolutions capture redundant information and a coarser resolution performs better. Therefore, we sample subgrids with different step-sizes from the discretized grid described in section~\ref{sec:algs} and vary $\vp$ over these subgrids.We first describe an experiment on a synthetic data set and follow it with experiments on benchmark graph data sets.

\begin{figure*}[!htb]
    \centering    \includegraphics[width=0.7\textwidth,page=2]{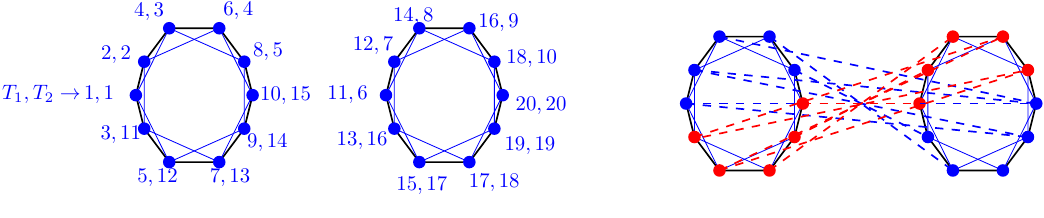}
    \caption{(Left) An example of a graph consisting of two circulant subgraphs. The pair of indices on each node represents the its order on the traversals $T_1$ and $T_2$ respectively. Both traversals start from the left node as the root node. (Right) Cross edges placed across two subgraphs. 
    }
    \label{fig:toy_dataset}
\end{figure*}

\subsection{Experiment with HourGlass dataset}
We test our model on a synthetic dataset (HourGlass) that entails a binary graph classification problem over a collection of attributed undirected graphs. Note that this synthetic dataset is designed to show that some attributed graphs can be easily classfied by $2$-parameter persistence modules but not so by $1$-parameter persistence modules or commonly used GNN models.
Each graph $G$ from either class is composed with two circulant subgraphs $G_1, G_2$ connected by some cross edges. The node attributes are order indices generated by two different traversals $T_1, T_2$. The label of classes corresponds to these two different traversals $T_1, T_2$. Therefore, the classification task is that given an attributed graph $G$, the model needs to predict which traversal is used to generate $G$. 
See Figure~\ref{fig:toy_dataset} (left) as an example of two attributed graphs with the same graph structure but with different
node attributes generated by two different traversals. More details can be found in Appendix~\ref{asec:toydata_exp}. We denote HourGlass[a,b] as the dataset of graphs generated with node size of each circulant subgraphs in range $[a,b]$. 
We generate three datasets with different sizes:  HourGlass[10,20],  HourGlass[21,30], HourGlass[31,40]. 
Each dataset contains roughly 400 graphs.
We evenly split HourGlass[21,30] into balanced training set and testing set on which we compare \gril{} with several commonly used GNN models from the literature including: Graph Convolutional Networks (GCN)~\citep{kipf2017semi}, Graph Isomorphism Networks (GIN)~\citep{gin_xu2018how} and a $1$-parameter persistent homology vector representation called persistence image (PersImg~\citep{AdamPersImg}. All GNN models contain $3$ aggregation layers. All models use 3-layer multilayer perceptron (MLP) as classifiers. More details about model and training settings can be found in Appendix~\ref{asec:toydata_exp}.
We also test these trained models on HourGlass[10,20] and HourGlass[31,40] to check if they can generalize well on smaller and larger graphs. The experimental results are shown in Table~\ref{tb:toy_dataset}. We can see that this dataset can be easily classified by our model based on $2$-parameter persistence modules with good generalization performance but $1$-parameter persistence method like PersImg or some GNN models struggle with this dataset.

\begin{table}[!tt]
\centering
\begin{tabular}{@{}ccccc@{}}
\toprule
 \multicolumn{5}{c}{\textbf{Testing accuracy of models on HourGlass}}\\
\midrule
\textbf{Model}                               & \textbf{GCN}  & \textbf{GIN}  & \textbf{PersImg} & \textbf{\gril{}} \\ 
\midrule
HourGlass[21,30]                             & \textcolor{blue}{87.25$\pm$4.0}  & 84.00$\pm$4.4  & 74.00$\pm$7.4   & \textcolor{red}{100.0$\pm$0.0}    \\ 
HourGlass[10,20]                             &  \textcolor{blue}{67.31$\pm$4.6} & 62.98$\pm$3.4   & 50.33$\pm$1.6  & \textcolor{red}{99.79$\pm$0.1}    \\ 
HourGlass[31,40]                             &  \textcolor{blue}{87.75$\pm$2.2} & 79.10$\pm$6.2   & 86.95$\pm$5.0  & \textcolor{red}{100.0$\pm$0.0}   \\ 
\bottomrule
\end{tabular}
\vspace*{0.2cm}
\caption{Table of testing results from different models. Last two rows show the testing results on HourGlass[10,20]  and HourGlass[31,40] of models trained on HourGlass[21,30]. For each dataset accuracies reported in \textcolor{red}{red} and \textcolor{blue}{blue} denote the best and second-best performance respectively.}
\label{tb:toy_dataset}
\end{table}
\begin{table}[!htbp]
    \centering
    \begin{tabular}{@{}cccccccc@{}}
    \toprule
         \textbf{Dataset} & \textbf{MP-I} & \textbf{MP-K} &\textbf{MP-L} & \textbf{P} & \textbf{\gril{}} \\ \midrule
         
         \textsc{Proteins} & 67.3 $\pm$ 3.5 & \textcolor{blue}{67.5 $\pm$ 3.1} & 65.8 $\pm$ 3.3 & 65.4 $\pm$ 2.7 & \textcolor{red}{70.9 $\pm$ 3.1} \\
         \textsc{Dhfr} & \textcolor{blue}{80.2 $\pm$ 2.3} & \textcolor{red}{81.7 $\pm$ 1.9} & 79.5 $\pm$ 2.3 & 70.9 $\pm$ 3.1 & 77.6 $\pm$ 2.5  \\
         \textsc{Cox}2 & 77.9 $\pm$ 2.7 & \textcolor{red}{79.9 $\pm$ 1.8} & 79.0 $\pm$ 3.3 & 76.0 $\pm$ 4.1 & \textcolor{blue}{79.8 $\pm$ 2.9} \\
         \textsc{Mutag} & 85.6 $\pm$ 7.3 & \textcolor{blue}{86.2 $\pm$ 2.6} & 85.7 $\pm$ 2.5 & 79.2 $\pm$ 7.7 & \textcolor{red}{87.8 $\pm$ 4.2 }\\
         \textsc{Imdb-Binary} & \textcolor{blue}{71.1 $\pm$ 2.1} & 68.2 $\pm$ 1.2 & \textcolor{red}{71.2 $\pm$ 2.0} & 54.0 $\pm$ 1.9 & 65.2 $\pm$ 2.6 \\ \bottomrule
    \end{tabular}
    \vspace*{0.2cm}
    \caption{Test accuracy of different models on graph datasets. The values of the MP-I, MP-K, MP-L and P columns are as reported in \citep{Carriere_Multipers_Images}. P denotes $1$-parameter persistence as reported in ~\citep{Carriere_Multipers_Images}.}
    \label{tab:graph_acc}
\end{table}


\begin{table}[!htbp]
    \centering
    \begin{tabular}{@{}cccccc@{}}
\toprule
\textbf{Model}      & \textsc{Proteins}      & \textsc{Dhfr}          & \textsc{Cox2}          & \textsc{Mutag}         & \textsc{Imdb-Binary}  \\ \midrule
\textbf{GCN}        & $71.15 \pm 2.31$ & $78.70 \pm 2.35$ & $78.80 \pm 2.13$ & $88.26 \pm 3.70$ & $73.1 \pm 2.20$ \\ 
\textbf{GCN + GRIL} & \textcolor{red}{$74.21 \pm 2.08$} & $75.66\pm3.08$   & \textcolor{blue}{$80.30 \pm 1.57$} & \textcolor{blue}{$88.80 \pm 3.60$} & $72.6 \pm 1.46$ \\ 
\midrule
\textbf{GAT}        &     $67.66 \pm 3.92$  & $77.78 \pm 4.50$            &  $79.45 \pm 3.68$              &   $86.69 \pm 6.36$            &     \textcolor{red}{$74.90 \pm 2.98$}         \\ 
\textbf{GAT + GRIL} & $71.60 \pm 3.92$ &  \textcolor{blue}{$79.64 \pm 6.29$}   &    \textcolor{red}{$80.52 \pm 3.30$} &    $84.03 \pm 7.85$ & $71.60 \pm 3.04$         \\ 
\midrule
\textbf{GIN}        & $69.09 \pm 3.77$ & \textcolor{red}{$79.77 \pm 6.72$} & $78.80 \pm 4.88$ & $83.97 \pm 6.04$ & $73.7 \pm 3.34$ \\ 
\textbf{GIN + GRIL} & \textcolor{blue}{$71.87 \pm 3.22$} & $78.46\pm5.80$   & $79.22 \pm 4.89$ & \textcolor{red}{$89.32 \pm 4.81$} & \textcolor{blue}{$74.2 \pm 2.82$} \\ \bottomrule
\end{tabular}
\vspace*{0.2cm}
\caption{Performance comparison of baseline GNNs and \textsc{GRIL} augmented GNNs on graph benchmark datasets.}
    \label{tab:GNN_aug}
\end{table}
\begin{table}[!htb]
\centering
\begin{tabular}{@{}ccccc@{}}
\toprule
Model      & IMDB-BINARY               & IMDB-MULTI               & REDDIT-BINARY               & REDDIT-MULTI-5K             \\ \midrule
           & \multicolumn{2}{c}{$\text{initial\_node\_features: deg(v)}$} & \multicolumn{2}{c}{$\text{initial\_node\_features: uninformative}$} \\
    \midrule
GIN        & $73.70 \pm 3.34$          & $49.60 \pm 3.02$         & $90.30 \pm 1.30$           & $53.77 \pm 1.85$            \\
GIN + GRIL & $74.20 \pm 2.82$          & $50.33 \pm 2.58$         & $87.35 \pm 2.77$            & $53.85 \pm 2.60$ \\
\bottomrule
\end{tabular}
\vspace*{0.2cm}
\caption{Performance comparison of baseline GNNs and \textsc{GRIL} augmented GNNs on social network datasets without node attributes.}
\label{tab:uninformative}
\end{table}

\subsection{Graph Experiments} 
We perform a series of experiments on graph classification to test the proposed model. We use standard datasets such as \textsc{Proteins, Dhfr, Cox2, Imdb-Binary} and \textsc{Mutag}~\citep{TUDatasets}. A quantitative summary of these datasets is given in Appendix \ref{app:section:graph_exp}. 
\subsubsection{Classifying \gril{} representations directly}  We compare the performance of \gril{} with other models such as multiparameter persistence landscapes (MP-L)~\citep{Multipers_landscapes}, multiparameter persistence images (MP-I)~\citep{Carriere_Multipers_Images}, multiparameter persistence kernel (MP-K)~\citep{Multipers_Kernel_Kerber}.

In~\citep{Carriere_Multipers_Images}, the authors use the heat kernel signature (HKS) and Ricci curvature to form a bi-filtration on the graph datasets. We also use the same bi-filtration and report the result in Table \ref{tab:graph_acc}. We use XGBoost classifier~\citep{xgboost} as done in~\citep{Carriere_Multipers_Images} for a fair comparison. We also report the results of \gril{} with different classifiers in Table \ref{tab:diff_clfs}. 
The reported accuracies are averaged over 5 train/test splits of the datasets obtained with 5 stratified folds. The full details of the experiments are given in Appendix \ref{app:section:graph_exp}.

From Table \ref{tab:graph_acc}, we can see that the performance of \gril{} on \textsc{Imdb-Binary} is slightly lower than the other methods. This is because the graphs in \textsc{Imdb-Binary} do not contain many cycles and hence, there is not enough information to capture in $H_1$ (See Appendix~\ref{asec:viz} for a visual interpretation). However, when there is information available, \gril{} captures it better than the existing methods as can be seen from the accuracy values on other datasets.

\subsubsection{Augmenting GNNs with \gril{} features}
\paragraph{Experimental Setup.} In another set of experiments, we augment standard GNNs with \gril{} features and compare the performance of the model with the existing ones. We use $3$ layers of message-passing with hidden dimensionality of $64$. The latent node representations are passed through a pooling layer and a two layer MLP to obtain the final classification. We use sum pooling to maintain uniformity among experiments and we do not claim that this is the optimal choice in any sense. For the GNN+\gril{} architectures, we concatenate $H_0$ and $H_1$ and pass it through a $1$-layer MLP. We concatenate the transformed \gril{} values with the graph-level representations obtained from the pooling layer before passing through the final MLP classifier.
\vspace*{-0.3cm}
\paragraph{Training and evaluation.} The models are trained for $100$ epochs with ADAM as the optimizer. The initial learning rate was set to be $10^{-2}$ halving every $20$ epochs. No hyperparameter tuning and early stopping was done. Though restrictive for practical scenarios, we follow earlier works (see \citep{morris2019weisfeiler},\citep{zhang2018end} for more details). We report cross-validation accuracy averaged over $10$ folds of the model obtained in the final training epoch.
\vspace*{-0.3cm}

\paragraph{Results.}We can see from Table~\ref{tab:GNN_aug} that \gril{} captures topological information that the GNN architectures are unable to capture and hence we see a clear increase in performance. However this is not the case for social network datasets. For the experiments reported in table~\ref{tab:uninformative} the node features are set as \textit{uninformative} following the settings of~\citep{gin_xu2018how}. For the \textsc{Imdb-*, Reddit-Multi-5K} datasets, the augmented \gril{} features improve the baseline GIN accuracy. For the \textsc{Reddit-Binary} dataset, since the graphs are highly sparse \gril{} features computed with HKS-RC bifiltration fails to capture important features and as a consequence, the performance decreases.

\section{Conclusion}
In this work, we propose \gril{}, a $2$-parameter persistence vectorization based on generalized rank invariant that we show is Lipschitz continuous and differentiable with respect to the bi-filtration functions. 
Further, we present an algorithm for computing \gril{} which is a synergistic confluence of the recent developments in computing generalized rank invariant of a $2$-parameter module and an efficient algorithm for computing zigzag persistence.
As a topological feature extractor, \gril{} performs better than Graph Convolutional Networks (GCNs) and Graph Isomorphism Networks (GINs) on our synthetic dataset. 
It also
performs better than the existing multiparameter persistence methods on some graph benchmark datasets while achieves comparable performance on others. 

We believe that the additional topological information that a $2$-parameter persistence module encodes, as compared to a $1$-parameter persistence module, can be leveraged to learn better representations. Further directions of research include using \gril{} with GNNs for filtration learning to learn more powerful representations. We expect that this work motivates further research in this direction.

\section{Acknowledgement}
This research is partially supported by NSF grant CCF 2049010.

\newpage
\bibliographystyle{iclr2023_conference}
\bibliography{ref}

\newpage
\appendix

\section{Algorithm}
\label{sec:algorithm}
Here, we describe the algorithm in detail. In practice, we are usually presented with 
a piecewise linear (PL) approximation $\hat{f}$ of a $\RR^2$-valued function $f$
on a discretized domain such as a finite simplicial complex. The PL-approximation $\hat{f}$ itself is $\RR^2$-valued. 
Discretizing the parameter
space $\RR^2$ by a grid, we consider a \emph{lower star}
bi-filtration of the simplicial complex.
Analogous to the
$1$-parameter case, a lower star bi-filtration is obtained by
assigning every simplex the maximum
of the values over all of its vertices in each of the two co-ordinates.
With appropriate scaling, these (finite) values can be mapped to a subset of points in a uniform finite grid over $[0,1]\times [0,1]$.
Observe that because of the maximization of values over all vertices,
we have the property that two simplices $\sigma\subseteq \tau$ 
have values $\hat{f}(\sigma)\in \RR^2$ and $\hat{f}(\tau)\in \RR^2$ where
$\hat{f}(\sigma)\leq \hat{f}(\tau)$.
A partial order of the simplices according to these values 
provide a bi-filtration over the grid $[0,1]\times [0,1]$.

{\bf Computing generalized ranks.} We need to compute the generalized rank $\RKG^M(\widehat{\boxed{\vp}}^\ell_{d})$ for
every worm $\widehat{\boxed{\vp}}^\ell_d$ to decide whether to increase its width or not. We use a result of~\citep{DKM22} to compute $\RKG^M(\widehat{\boxed{\vp}}^\ell_{d})$. It says that $\RKG^M(\widehat{\boxed{\vp}}^\ell_{d})$ can be
computed by considering a zigzag module and computing
the number of full bars (bars that begin at the start of the zigzag filtration and persist until the end of the filtration)  in its decomposition. This zigzag
module decomposition can be obtained by restricting
the bi-filtration on the boundary of $\RKG^M(\widehat{\boxed{\vp}}^\ell_{d})$ and using
any of the zigzag persistence algorithms on the resulting
zigzag filtration. We use the recently published efficient algorithm and its associated software~\citep{dey2022fast} for computing zigzag persistence.\\

{\bf Computing the value of \gril{} using binary search.} For a worm $\widehat{\boxed{\vp}}^\ell_{d}$ and a given $k\geq 1$, we apply binary search to compute the value of \gril{}. 

Let us denote the grid resolution by $\rho$. We do the binary search for $d$ in the range $[d_{\min}, d_{\max}]$ where $d_{\min}=\rho$  and $d_{\max}=1$. In each iteration, we compute $\RKG^M(\widehat{\boxed{\vp}}^\ell_d)$ for $d=(d_{\min}+d_{\max})/2$ and check if $\RKG^M(\widehat{\boxed{\vp}}^\ell_d)\geq k$. We increase the width of the worm by updating 

$d_{\min}$ to be $d+\rho$ if $\RKG^M(\widehat{\boxed{\vp}}^\ell_d)\geq k$. Otherwise, we decrease the width of the worm by updating $d_{\max}$ to be $d-\rho$. The binary search stops and returns $d$ when $d_{\max} < d_{\min}$. This ensures that we have searched through all possible values of $d$ for which $\RKG^M(\widehat{\boxed{\vp}}^\ell_d) \geq k$ and returned the maximum of these values.

\begin{figure}[thbp]
\begin{center}
    \includegraphics[width=0.7\textwidth]{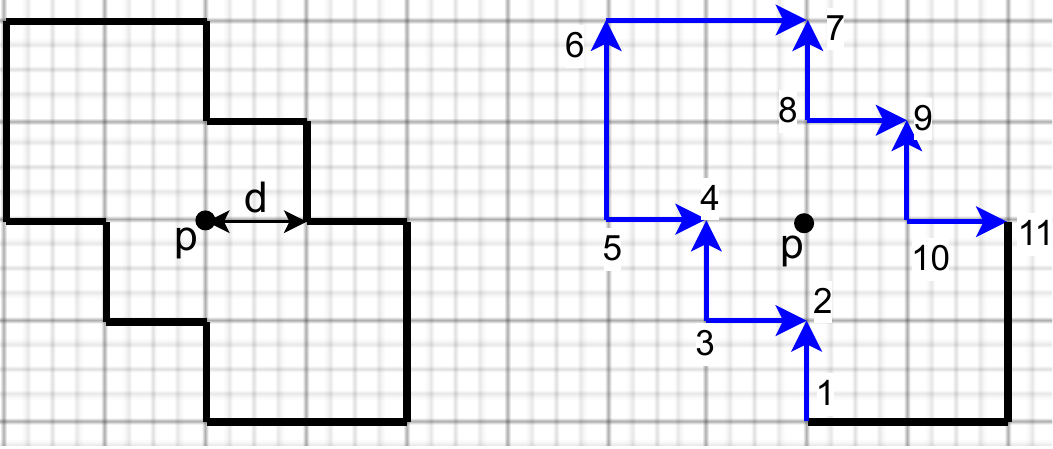}
    \caption{(Left) The figure shows the $2$-worm centered at $p$ with width $d$. (Right) The highlighted part denotes the boundary cap of the worm. The arrows in the figure denote the direction of arrows in the zigzag filtration.}
    \label{fig:worm_bdry}
\end{center}
\end{figure}

Refer to Figure \ref{fig:worm_bdry} for an illustration of the zigzag filtration along the boundary cap of a $2$-worm.

\section{Experimental Setup}
\label{app:experiment_setup}
\subsection{Hourglass Dataset}\label{asec:toydata_exp}
The two traversals $T_1$ and $T_2$ are designed as follows: $T_1$ traverses $G_1$, then followed by $G_2$; $T_2$ traverses upper halves $G^\top_1\subseteq G_1$ and $G^\top_2\subseteq G_2$ sequentially first, then followed by the other halves $G^\bot_1\subseteq G_1$ and $G^\bot_2\subseteq G_2$.
For cross edges, we randomly pick $2|V|$ pairs of nodes (with replacement) in $G^\top_{1}\times G^\bot_{2}$ on which we place cross edges. We don't place multiple edges on the same pair of nodes. In a similar way we place cross edges on $G^\bot_{1}\times G^\top_{2}$. Therefore, $G$ has roughly $6|V|$ cross edges between $G_1$ and $G_2$. The (roughly) total number of edges: $|E|\approx 5|V|$. For methods based on persistence modules, we take two filtration functions $f_1, f_2:V\cup E\to \RR$ on $G$ as follows: 
let $x(v)$ be the node attribute on $v$ given by the order index of the trace. Then
\begin{itemize}
    \item $f_1$ is given by $\forall v\in V, f_1(v)=x(v)$ and $\forall e=(v,w)\in E, f_1(e)=\max(x(v), x(w))$. 
    \item $f_2$ is given by $f_2(v)=0$ and $f_2=C(e)$ where $C(e)$ is a curvature value of $e$. Here we use a version of discrete Ricci called Forman-Ricci curvature~\citep{forman_ricci2003bochner} computed by the code provided in~\citep{ricci_flow}.
\end{itemize} 
We compute for all points $\vp$ in a uniform $4\times 4$ grid the {\gril{}} values $\lambda(\vp, k, \ell)$ 
for generalized rank $k = 1, 2$, worm size $\ell = 2$, and homology of dimension $0$ and $1$.
Therefore, for each graph our {$\Lambda_\vp^{k,\ell}$} generates a $64$-dimensional vector as representation. 
For the method based on $1$-parameter persistence modules with persistence image vectorization, we compute $1$-parameter persistence modules for homology dimension $0,1$ on $f_1$ and $f_2$ independently. Each persistence module will be vectorized on a $4\times 4$ grid. Therefore, it also produces a $64$-dimensional vector as representation. 


\subsection{Graph Experiments} 
\label{app:section:graph_exp}
We performed a series of experiments on graph classification using \gril{}. We used standard datasets with node features such as \textsc{Proteins, Dhfr, Cox2, Mutag} and \textsc{Imdb-Binary} \citep{TUDatasets}. Description of the graph classification tasks is given in Table \ref{tab:app:graph_data}.

\begin{table}[h!]
    \centering
    \caption{Description of Graph Datasets}
    \label{tab:app:graph_data}
    \vspace{5mm}
    \resizebox{0.7\columnwidth}{!}{
    \begin{tabular}{|c|c|c|c|c|}
    \hline
         \textbf{Dataset} & \textbf{Num Graphs} & \textbf{Num Classes} & \textbf{Avg. No. Nodes} & \textbf{Avg. No. Edges} \\
         \hline
         {\sc Proteins}& 1113 & 2 & 39.06 & 72.82 \\
         {\sc Cox2} & 467 & 2 & 41.22 & 43.45 \\
         {\sc Dhfr} & 756 & 2 & 42.43 & 44.54 \\
         {\sc Mutag} & 188 & 2 & 17.93 & 19.79 \\
         {\sc Imdb-Binary} & 1000 & 2 & 19.77 & 96.33\\
         \hline
    \end{tabular}}
    
\end{table}

The Heat Kernel Signature-Ricci Curvature bi-filtration, as done in \citep{Carriere_Multipers_Images}, values are normalized so that they lie between $0$ and $1$. For the experiments reported in Section \ref{sec:experiment}, we fix the grid resolution $\rho = 0.01$. Thus, the square $[0,1] \times [0,1]$ has $100 \times 100$ many grid points. We sample a uniform subgrid of center points, $\vp$, out of these grid points. We fix $l = 2$ for our experiments. We compute $\lambda(\vp, k, \ell)$ where $\vp$ varies over the sampled center points and $k$ varies from $1$ to $5$. Each such computation is done for dimension $0$ homology ($H_0$) and dimension $1$ homology ($H_1$). We use XGBoost~\citep{xgboost} classifier for these experiments.

\subsubsection{Ablation Studies.} 
We have performed experiments with different subgrid sizes and the results are reported in Table \ref{tab:grid_size}. The reported accuracies are averaged over 5 train/test splits of the datasets obtained with 5 stratified folds. We can see from the table that for different datasets, different subgrid sizes give the best results. This can be attributed to the fact that for some datasets, topological information needs to be captured at a finer level while for other datasets, capturing such finer details can be redundant. 

\begin{table}[!htb]
    \centering
    \begin{tabular}{@{}ccccc@{}}
        \toprule
         \textbf{Grid Size} & $\mathbf{50\times50}$ & $\mathbf{25\times 25}$ & $\mathbf{10\times 10}$ & $\mathbf{5\times 5}$  \\
         \midrule
         \textsc{Proteins} & $70.8 \pm 2.7$ & $70.2\pm 1.8$ & $69.8\pm 2.4$ & $68.5 \pm 2.6$  \\
         \textsc{Dhfr} & $77.6\pm 2.5$ & $77.2\pm3.4$ & $77.5\pm 3.5$ & $77.5 \pm 3.5$ \\
         \textsc{Cox2} & $79.8 \pm 3.0$ & $78.9 \pm 2.4$ & $79.8 \pm 2.9$ & $78.9 \pm 3.5$ \\
         \textsc{Mutag} & $87.3\pm3.8$ & $87.8\pm 4.2$ & $87.8\pm4.5$ & $86.8\pm 3.3$ \\
         \textsc{Imdb-Binary} & $62.2 \pm 4.3$ & $65.2 \pm 2.6$ & $62.2 \pm 2.3$ & $63.5 \pm 3.2$ \\
         \bottomrule
    \end{tabular}
    \vspace*{0.2cm}
    \caption{Test accuracies of \gril{} on subgrids of different sizes.}
\label{tab:grid_size}
\end{table}

We report the computation times of \gril{} for these datasets in Table \ref{tab:comp_times}. The values denote the total computation time for all the center points on a $50\times 50$ subgrid for a $2$-worm. The computations were done on a Intel(R) Xeon(R) Gold 6248R CPU machine and the computation was carried out on 32 cores.

\begin{table}[!htb]
    \centering
    
    \begin{tabular}{@{}cc@{}}
    \toprule
         \textbf{Dataset} & \textbf{Computation time}\\
         \midrule
         \textsc{Proteins} & $6$ hr $13$ min $38$ s \\
         \textsc{Dhfr} & $4$ hr $15$ min $54$ s \\
         \textsc{Cox2} & $2$ hr $44$ min $23$ s \\
         \textsc{Mutag} & $0$ hr $56$ min $48$ s \\
         \textsc{Imdb-Binary} & $4$ hr $03$ min $35$ s \\
         \bottomrule
    \end{tabular}
    \vspace*{0.2cm}
    \caption{Computation times for \gril{} for each dataset with a $2$-worm and $50 \times 50$ subgrid. }
    \label{tab:comp_times}
\end{table}

In Table \ref{tab:grid_res_acc}, we show the performance of \gril{} with different grid resolutions ($\rho$) and $\ell$-worms. For these experiments, we used a $50\times 50$ subgrid for the center points. The reported accuracies are averaged over 5 train/test splits of the datasets obtained with 5 stratified folds. We test it on \textsc{Mutag} and \textsc{Cox2} and we can see that for $\rho=0.01$, we get the highest accuracy of the model on both the datasets. We can see from the table that there is an improvement in accuracy from $\ell=1$ to $\ell=2$. However, there is no significant improvement from $\ell=2$ to $\ell=3$.

\begin{table*}[!htb]
    \centering
    \begin{tabular}{@{}ccccccc@{}}
        \toprule
        \textbf{Dataset} & $\mathbf{\rho = 0.02}$ & $\mathbf{\rho = 0.01}$ & $\mathbf{\rho = 0.005}$ & $\mathbf{\ell = 1}$ & $\mathbf{\ell = 2}$ & $\mathbf{\ell = 3}$\\ 
        \midrule
        \textsc{Mutag} & 86.3 $\pm$ 4.2 & 87.8 $\pm$ 4.5 & 85.2 $\pm$ 3.9 & 85.7 $\pm$ 4.2 & 87.8 $\pm$ 4.5 & 87.8 $\pm$ 3.9 \\ 
        \textsc{Cox2} & 78.2 $\pm$ 1.7 & 79.8 $\pm$ 2.9 & 77.8 $\pm$ 1.4 & 79.3 $\pm$ 2.9 & 79.8 $\pm$ 2.9 & 78.9 $\pm$ 3.5\\
        \bottomrule
    \end{tabular}
    \caption{Test accuracy for different grid resolutions and for $\ell$-worms with different  values of $\ell$.}
\label{tab:grid_res_acc}
\end{table*}

         
    

\begin{table}[!ht]
    \centering
    \begin{tabular}{@{}ccccc@{}}
        \toprule
         \textbf{Dataset} & \textbf{SVM} & \textbf{LR} & \textbf{XGBoost} & \textbf{3-MLP} \\
         \midrule
         \textsc{Proteins} & $73.3 \pm 1.5$ & $72.7\pm 2.6$ & $70.9 \pm 3.1$ & $71.3 \pm 2.1$ \\
         \textsc{Dhfr} & $61.7\pm 0.4$ & $77.8\pm1.9$ & $77.6 \pm 2.5$ & $72.3 \pm 4.3$ \\
         \textsc{Cox2} & $77.2 \pm 0.8$ & $78.5 \pm 2.5$ & $79.8 \pm 2.9$ & $77.0 \pm 1.2$ \\
         \textsc{Mutag} & $80.0\pm3.9$ & $86.3\pm 3.8$ & $87.8\pm 4.2$ & $76.8 \pm 9.1$ \\
         \textsc{Imdb-Binary} & $65.1 \pm 3.6$ & $63.2 \pm 2.1$ & $65.2 \pm 2.6$ & $61.2 \pm 6.6$ \\
         \bottomrule
    \end{tabular}
    \vspace*{0.2cm}
    \caption{Test accuracies of \gril{} using different classifiers.}
    \label{tab:diff_clfs}
\end{table}
In Table \ref{tab:diff_clfs}, we report the performance of \gril{} on graph benchmark datasets with different classifiers such as Support Vector Machine (SVM)~\citep{SVM, LIBSVM}, Logistic Regression (LR)~\citep{LIBLINEAR}, Multilayer Perceptron (3-MLP) implemented using \emph{scikit-learn}~\citep{sklearn_api} library.
The reported accuracies are averaged over 5 train/test splits of the datasets obtained with 5 stratified folds.

\section{Differentiability of {\gril{}}: proof and experiment}\label{asec:diff}

\begin{propositionof}{\ref{thm:derivative}}
    Consider the space of all filtration functions $\{f:\gX\to \RR^2\}$ on a finite space $\gX$ with $|\gX|=n$, which is equivalent to $\RR^{2n}$. 
   
    For fixed $k, \ell, \vp$, 
   there exists a measure-zero subset $Z\subseteq \RR^{2n}$ such that for any $f\in \RR^{2n}\setminus Z$ satisfying the following generic condition: $\forall x\neq y\in \gX, f(x)_1\neq f(y)_1, f(x)_2\neq f(y)_2$, there exists an assignment $s:\gX\to \{\pm 1,0,\pm\ell\}^2$ such that 
    \begin{align*}
        \nabla_s \Lambda_\vp^{k,\ell}(f)\triangleq & \lim_{\alpha\to 0}\frac{\Lambda_\vp^{k,\ell}(f+\alpha s)-\Lambda_\vp^{k,\ell}(f)}{\alpha\|s\|_{\infty}}\\
        = & \max_{g\in\gX}\nabla_g\Lambda_\vp^{k,\ell}(f).
    \end{align*}
\end{propositionof}

    \begin{proof}
    By Corollary~\ref{cor:gril_differentiable} we know there exists some measure-zero set $R\subset\RR^{2n}$ such that $\Lambda_\vp^{k,\ell}$ is differentiable in $\bar{R}\triangleq \RR^{2n}\setminus R$.
    Let $M=M^f$ be a $2$-parameter persistence module induced from some generic filtration function $f\in \bar{R}$ and $I=\boxed{\vp}^{\ell}_d$ be an $\ell$-worm in $\RR^2$ centered at some point $\vp$. Let $\partial(I)$ be the boundary of $I$ excluding the right most vertical edge and bottom most horizontal edge (See Figure~\ref{fig:worm_expand} as an illustration). It is shown in~\citep{DKM22} that, over the boundary $\partial(I)$, a  zigzag persistence module can be defined by restricting $M$ to $\partial(I)$ (in practice it is enough to take a zigzag path to approximate the smooth off-diagonal boundary) on which the number of full bars is equal to $\RKG^M(I)$. Let $I'=\boxed{\vp}^\ell_{d'}$ be another $\ell$-worm centered at $\vp$ for some $d'\neq d$. One can observe that, if the zigzag filtrations on $\partial(I)$ and $\partial(I')$ have the same order of insertion and deletion of simplices, then the number of full bars on $M|_{\partial(I)}$ and $M|_{\partial(I')}$ are the same, which means $\RKG^M(I)=\RKG^M(I')$. Now let $d=\lambda^{M}(k,\ell, \vp), I=\boxed{\vp}^\ell_d, I_-=\boxed{\vp}^\ell_{d-\varepsilon}, I_+=\boxed{\vp}^\ell_{d+\varepsilon}$ for some small enough $\varepsilon$. 
Based on the definition of $\lambda^{M}$, we know that $\RKG^M(I_-)\geq k$ and $\RKG^M(I_+)<k$, which means that zigzag filtrations change on some simplices while moving from $\partial(I_-)$ to $\partial(I_+)$. Either the collection of simplices changes or the order of simplices changes. The former case corresponds to the simplices with $x$ or $y$-coordinate aligned with some vertical or horizontal edges on $\partial(I)$. The latter case corresponds to those pairs of simplices $(\sigma, \tau)$ such that $f(\sigma)\vee f(\tau)\triangleq (\max(f(\sigma)_1, f(\tau)_1), \max(f(\sigma)_2, f(\tau)_2)$ is on some off-diagonal edges on $\partial(I)$. By the generic condition of the filtration function $f$, we can locate those simplices as the set $S$, which we call support simplices. The assignment function $s$ is defined on each $\sigma\in S$ by assigning $s(\sigma)=\pm 1$ or $\pm\ell$ which is consistent with the moving direction of the edge from $\partial(I)$ to $\partial(I_+)$. We discuss the assignment values case by case:

We can divide the boundary into four edges: bottom (off-diagonal) edge $e_b$, top (horizontal) edge $e_t$, left (vertical) edge $e_l$, right (off-diagonal) edge $e_r$.
\begin{enumerate}
    \item $ s(\sigma)=(0, +\ell)$ if $\sigma$ has $y$-coordinate the same as $e_t$, 
    \item $ s(\sigma)=(-\ell, 0)$ if $\sigma$ has $x$-coordinate the same as $e_l$, 
    \item $ s(\sigma)=(0,-1), s(\tau)=(-1, 0)$ if $f(\sigma)\vee f(\tau)$ is on $e_b$ and $f(\sigma)_1\leq f(\tau)_1$, 
    \item $ s(\sigma)=(0,+1), s(\tau)=(+1, 0)$ if $f(\sigma)\vee f(\tau)$ is on $e_r$ and $f(\sigma)_1\leq f(\tau)_1$,     
\end{enumerate}
See Figure~\ref{fig:worm_expand} as an illustration. 
We assume $f$ satisfies the condition that the supporting simplices in $S$ either all belong to cases $1$ and $2$ or all belong to cases $3$ and $4$, but not a combination of them. It is not hard to see that the collection of $f$ for which this condition does not hold is a measure zero set in $\RR^{2n}$. Let us denote the collection of all such $f$'s by $F$. Then, $Z = F \cup R$ is a measure zero set in $\RR^{2n}$ which consists of $f$'s which do not satisfy the condition and those points where $\Lambda_\vp^{k,\ell}${} is not differentiable. 

Now, check for such a generic $f\notin Z$ so that the directional derivative $\nabla_s\lambda(f)$ is indeed a maximal directional derivative. 
For the cases $3$ and $4$,
the stability property in Proposition~\ref{prop:stability} implies that, for any $\alpha>0$  and any direction vector $g\in \Real^{2n}$ with $\|g\|_\infty=1$, we have $\lambda(f+\alpha g)-\lambda(f)\leq \alpha$. Also it is not hard to check that 
$\lambda(f+\alpha s)-\lambda(f)=\alpha$ for $\alpha>0$ small enough 
since the zigzag persistence of $M^{f+\alpha s}|_J$ with $J=\boxed{\vp}^\ell_{d+\alpha}$ has the same collection of simplices and orders as $M^f|_{I}$ with $I=\boxed{\vp}^\ell_{d}$, which means they have the same rank. Therefore, we have $\forall \|g\|_\infty=1, \lambda(f+\alpha g)-\lambda(f)\leq \lambda(f+\alpha s)-\lambda(f) \implies \nabla_g\Lambda(f)\leq \nabla_s\Lambda(f)$. 
For the case $1$ (the case $2$ is similar), the support simplex is on edge $e_t$. Now for any direction vector $g\in \RR^{2n}$ and $\alpha>0$  small enough, let $\Delta d=\Lambda(f+\alpha g)-\Lambda(f)$ and let $\Delta y_{e_t}$ be the difference between $y$-coordinates of $e_t$'s from $\boxed{\vp}^\ell_d$ and  $\boxed{\vp}^\ell_{d+\Delta d}$. Note that $\frac{\Delta d}{\Delta y_{e_t}}=\ell$ and 
$\frac{|\Lambda(f+\alpha g)-\Lambda(f)|}{\alpha \|g\|_{\infty}}\leq \frac{\Delta d}{\Delta y_{e_t}}$ since in order to change $\Lambda(f)$ by $\Delta d$ one has to at least move edge $e_t$ by $\Delta y_{e_t}$, which correspondingly changes the $y$-coordinate of $s(\sigma)$ by $\Delta y_{e_t}$. From the above argument, we can get 
the directional derivative $\nabla_g\Lambda(f)$ is bounded from above by the ratio $\frac{\Delta d}{\Delta y_{e_t}}=\frac{1}{\ell}=\nabla_s\Lambda(f)$. 
The case for $\alpha < 0$ is symmetric. 

\begin{figure}[htb]
    \centering
    \includegraphics[width=0.7\textwidth,page=1]{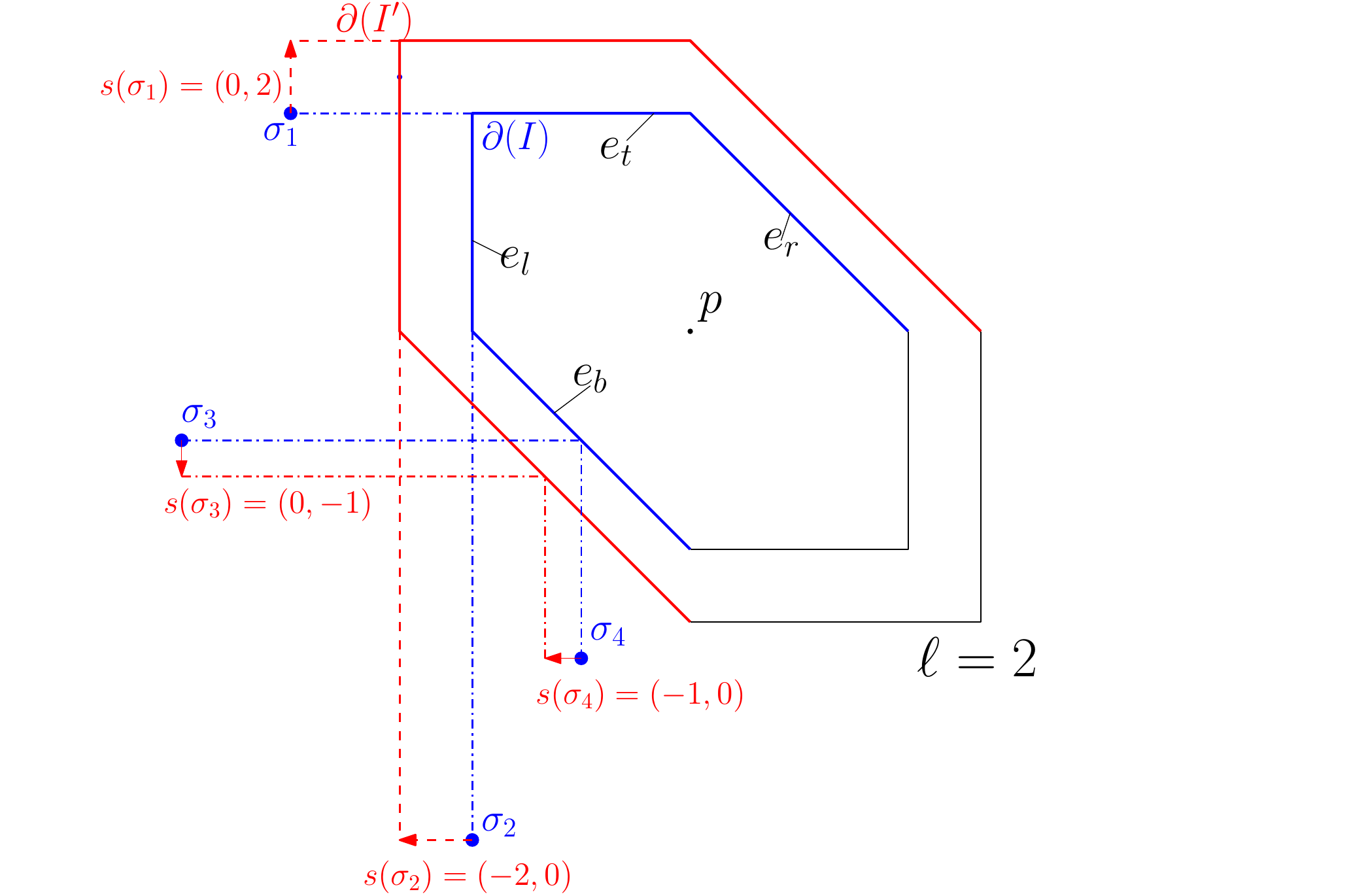}
    \caption{Two examples of $2$-worm $I, I'$. Blue and red lines are boundaries of $I$ and $I'$ respectively on which the zigzag persistence modules are constructed for computing ranks. $\sigma_i, i=1,2,3,4$ are four support simplices on $\partial(I)$. $s(\sigma_i)$ is the assignment function values on $\sigma_i$. }
    \label{fig:worm_expand}
\end{figure}

In summary, $\nabla_s\lambda(f)$ indeed maximizes the directional derivative for $f$. 
\end{proof}

The proof of Theorem~\ref{thm:derivative} also shows how to find the assignment $s$ with the corresponding set of \emph{supporting} simplices. This result enables us to update the filtration function of the simplices according to some target function based on $\Lambda_\vp^{k,\ell}$. Here we introduce an experiment, as a proof of concept, to show how one can use {\gril{}} as a machine learning model to enhance topological features.
By giving a suitable target function, our model is trained to rearrange the positions of input points to
better represent circles.
The experiment results is shown in Figure \ref{fig:two_circles}.
The input to $\Lambda_\vp^{k,\ell}$ is points sampled non-uniformly from two circles. Recall that \gril{} is defined over a $2$-parameter persistence module induced by some filtration function $f = (f_x, f_y)$. For every vertex $v$, we assign $f_x(v) = 1 - \exp(\frac{1}{\alpha}\sum_{i=1}^{\alpha} d(v, v_i))$, where $v_i$ denotes $i$-th nearest neighbor of the vertex $v$ and $d(v, v_i)$ denotes the distance between $v$ and $v_i$. 
For our experiments we fix $\alpha=5$. We set $f_y(v) = 0$. We compute \textsc{AlphaComplex} filtration~\citep{edelsbrunner2010computational} of the points and for each edge $e:=(u,v)$ we assign $f_x(e) = \max (f_x(u), f_x(v))$ and $f_y(e) = 1 - \exp(d(u,v))$. To obtain a valid bi-filtration function on the simplicial complex we extend the bi-filtration function from $1$-simplices to $2$-simplices, i.e. triangles. We pass $f$ as an input to $\Lambda_\vp^{k,\ell}$, coded with the framework \textsc{PyTorch}~\citep{pytorch}, that computes persistence landscapes. $\Lambda_\vp^{k,\ell}$ uniformly samples $n$ center points from the grid $[0, 1]^2$. 
Since \gril{} value computation can be done independently for each $k$ and a center point, we take advantage of parallel computation and implement the code in a parallel manner. In the forward pass we get \gril{} values 
$\lambda(\vp, k, \ell)$ for generalized rank $k = 1, 2$, worm size $\ell = 2$ and homology of dimension $1$ while varying $\vp$ over all the sampled center points. After we get the \gril{} values, we compute the assignment $s$ according to Theorem~\ref{thm:derivative}. During the backward pass, we utilize this assignment to compute the derivative of {$\Lambda_\vp^{k,\ell}$} with respect to the filtration function and consequently update it. We get $n$ values of $\lambda(\cdot, 1,2)$ for $n$ center points. We treat these $n$ values as a vector and denote is as $\lambda_1$. Similarly, we use $\lambda_2$ to denote the vector formed by values $\lambda(\cdot, 2, 2)$. 
We minimize the loss  $L=- (\|\lambda_{1}\|_2^2 + \|\lambda_{2}\|_2^2)$.  Figure \ref{fig:two_circles} shows the result after running {$\Lambda_\vp^{k,\ell}$} for $200$ epochs. The optimizer we use to optimize the loss function is Adam~\citep{AdamKingma} with a learning rate of $0.01$.
\begin{figure}[!hbt]
    \centering
    \includegraphics[scale=0.8]{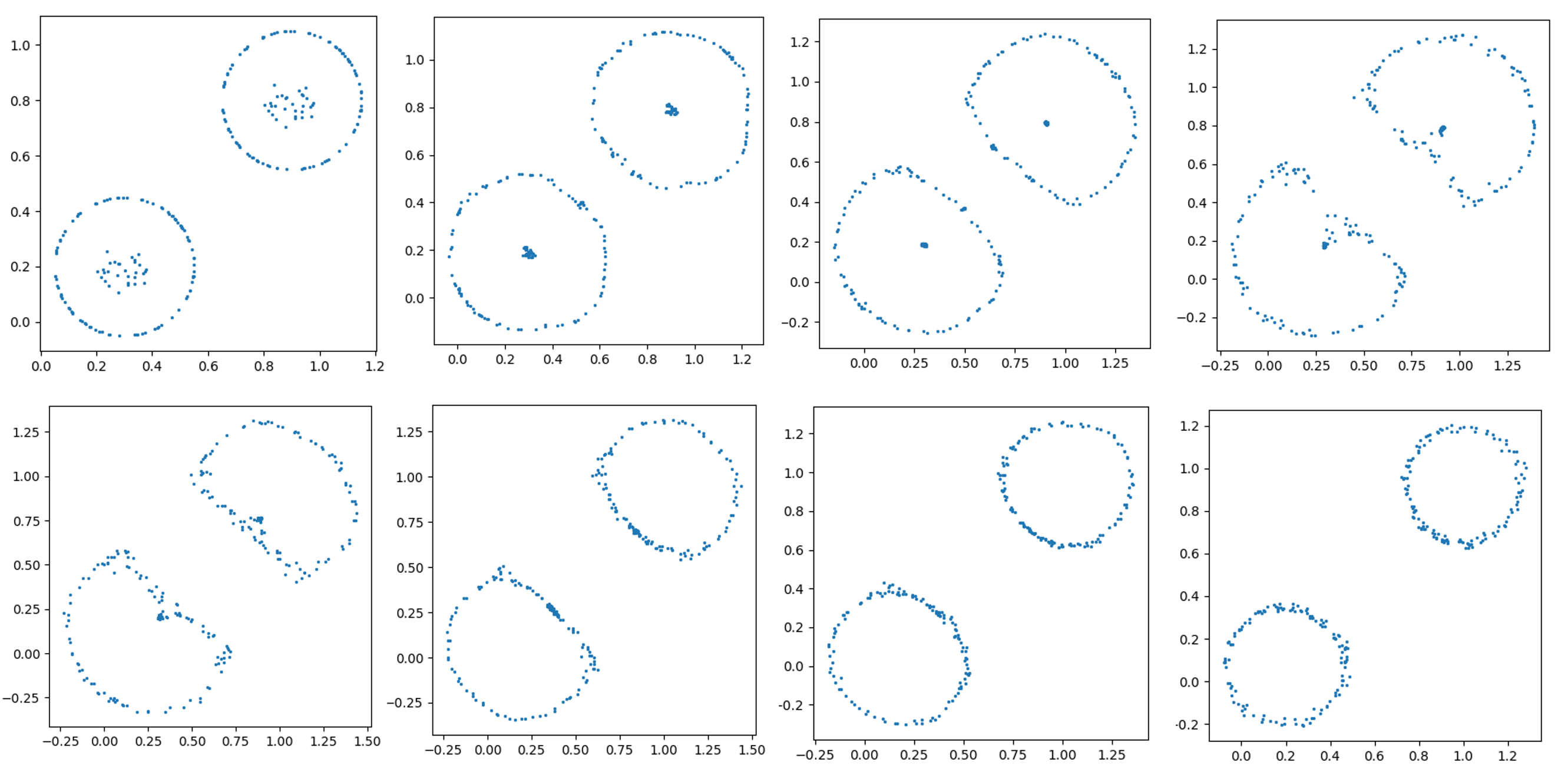}
    \caption{The figures show the rearrangement of points according to the loss function, which in our case is increasing the norm of $\lambda_1$ and $\lambda_2$ vectors. We start with two circles containing some noisy points inside. We observe that the points rearrange to form two circles because that increases the norm of $\lambda_1$ and $\lambda_2$ vectors.}
    \label{fig:two_circles}
\end{figure}
\newpage
\section{Visualization of \gril{} for graph datasets}\label{asec:viz}
The plot for first $5$ \gril{} values are shown in Figure~\ref{afig:sample_landscapes}. The figure contains landscape values for $5$ random graph samples of each dataset. In Figure~\ref{afig:sample_eigenvec}, we plot the first two eigen vectors given by principal component analysis (PCA) of the computed \gril{} values for each dataset. Plots for $H_0$ and $H_1$ are shown separately.
\begin{figure}[!htb]
    \centering
    \includegraphics[width=0.95\linewidth]{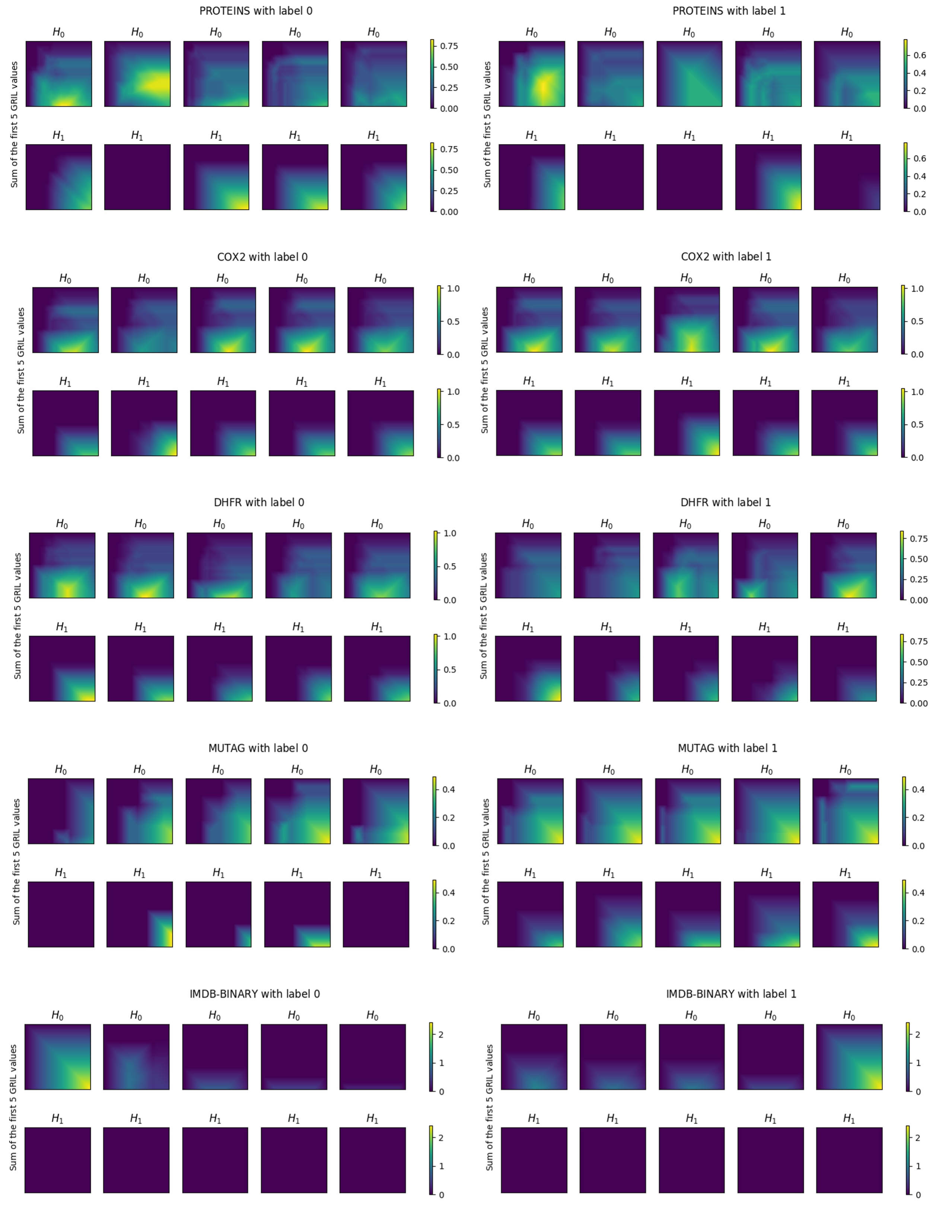}
    \caption{\gril{} of $5$ random graph samples of each dataset. \gril{} values of $H_0$ and $H_1$ are shown separately columnwise.}
    \label{afig:sample_landscapes}
\end{figure}

\begin{figure}[!htb]
    \centering
    \includegraphics[width=0.95\linewidth]{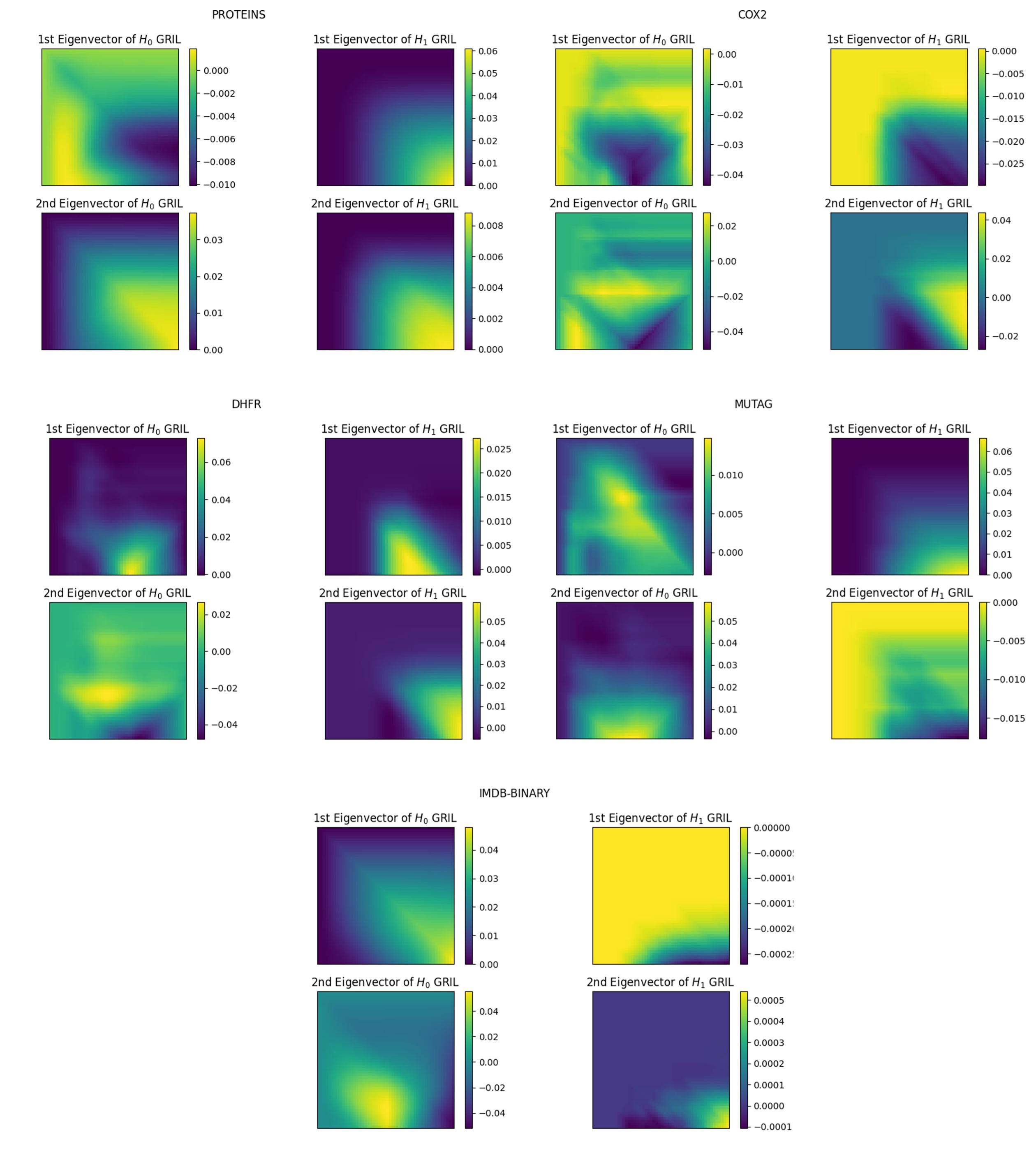}
    \caption{Plot of the first two eigen vectors given by PCA on the entire dataset for $H_0$ and $H_1$ respectively.}
    \label{afig:sample_eigenvec}
\end{figure}
\end{document}